\documentclass{article}

\PassOptionsToPackage{numbers}{natbib}
\usepackage[main, final]{neurips_2026}

\usepackage[utf8]{inputenc} %
\usepackage[T1]{fontenc}    %
\usepackage{hyperref}       %
\usepackage{url}            %
\usepackage{booktabs}       %
\usepackage{amsfonts}       %
\usepackage{nicefrac}       %
\usepackage{microtype}      %
\usepackage{xcolor}         %
\usepackage{amsmath}
\usepackage{amssymb}
\usepackage{tikz}
\usetikzlibrary{arrows.meta, positioning, calc}
\usepackage{multirow}
\usepackage{colortbl}
\usepackage{makecell}
\usepackage{float}
\usepackage{caption}
\usepackage{mathtools}
\usepackage{wrapfig}
\usepackage{algorithm}
\usepackage{algpseudocode}

\algrenewcommand{\algorithmiccomment}[1]{\hfill\makebox[6cm][l]{$\triangleright$ #1}}

\usepackage{graphicx}
\usepackage{enumitem}

\renewcommand{\vec}[1]{\boldsymbol{#1}}
\newcommand{\mat}[1]{\boldsymbol{#1}}

\newcommand{\ourmethod}{MOSEL}
\newcommand{\obj}{\mathcal{L}}

\usepackage{amsthm}
\newtheorem{theorem}{Theorem}
\theoremstyle{remark}
\newtheorem{remark}{Remark}
\theoremstyle{proposition}
\newtheorem{proposition}{Proposition}

\title{Learning Pareto Stationary Fronts \\ via Single-Pass Backpropagation}

\author{%
  Elina Rojin Celik \\
  Department of Computer Science\\
  Saarland University, Germany \\
  \texttt{celik@cs.uni-saarland.de} \\
  \And
   Marcos Medeiros Raimundo \\
   Instituto de Computa\c{c}\~{a}o \\
   Universidade Estadual de Campinas, Brazil\\
   \texttt{mraimundo@ic.unicamp.br} \\
   \AND
   Isabel Valera \\
   Department of Computer Science\\
   Saarland University, Germany \\
   \texttt{ivalera@cs.uni-saarland.de} \\
}

\begin{document}

\maketitle

\begin{abstract}
We propose \ourmethod\ (Multi-Objective Stackelberg Efficient Learning), a framework for a posteriori multi-objective optimization (MOO) in deep neural networks that recovers a full front of Pareto stationary solutions at the computational cost of standard single-objective training.
MOSEL reformulates the problem as a bilevel optimization problem that leverages network modularity to decouple representation learning from objective-preference alignment. Casting the bilevel problem as a Stackelberg game enables solving the original a posteriori MOO problem in a single forward–backward pass. %
As a result, \ourmethod\ matches the time and memory efficiency of standard single-objective training while enabling scalable Pareto stationary front learning. %
Empirically, \ourmethod\ uncovers diverse and optimal Pareto frontiers in strongly conflicting settings (e.g., fairness–accuracy). Remarkably, even in weakly conflicting regimes such as multi-task learning, it consistently converges to solutions closer to the utopia point, outperforming both standard single-objective training and specialized multi-task learning methods. 
These results highlight the broader potential of a posteriori MOO learning as a pathway to efficiently learn more diverse and robust representations, ultimately improving generalization.
\end{abstract}

\section{Introduction}
Multi-objective optimization (MOO) is ubiquitous across deep learning applications, arising whenever models must balance multiple objectives (i.e., losses during training). 
Many such objectives stem from competing stakeholder demands, leading to \emph{strongly conflicting objectives}. For example,  in \emph{fairness}, individual and group fairness must be traded-off against predictive performance~\cite{padh2021addressing}; and,  in \emph{LLM alignment}, objectives such as helpfulness, harmlessness, and honesty can conflict~\cite{he2025pareto, li2025gradient, pmlr-v286-ren25a}.
Similar conflicts emerge from the distributed or sequential nature of learning: in \emph{federated learning} each client optimizes a local objective that may conflict with those of other clients~\cite{yang2023federated};  and in \emph{reinforcement learning}, agents must balance multiple reward signals simultaneously~\cite{hayes2022practical, liu2025pareto, pmlr-v235-zhou24h}.  %
Other MOO applications lead to \emph{weak conflicts} among objectives where,  while the objectives may not be directly in conflict, the local nature of neural network optimization may lead to suboptimal solutions. 
For example,   in \emph{multi-task learning} (MTL)  joint training can improve performance on individual tasks compared to training them independently (positive transfer);  yet, task-specific losses may conflict during optimization,  %
degrading overall performance~\cite{liu2021conflict,sener2018multi,  yu2020gradient}. 
Similar weak conflicts arise in \emph{multimodal learning}%
~\cite{javaloy2022impartial, wei2024mmpareto}, and in \emph{domain generalization}%
~\cite{chen2023pareto}, often leading to suboptimal neural networks that fail to capture diversity across modalities and domains.

In practice, multiple objectives are often aggregated into a single scalar loss, a strategy that simplifies training but obscures trade-offs and can lead to suboptimal solutions. In particular, collapsing objectives into a single loss:
(i) in strongly conflicting regimes, fails to capture the full spectrum of Pareto optimal solutions, thereby masking the trade-offs among objectives~\cite{hayes2022practical, roijers2013survey}; or, 
(ii) in weakly conflicting regimes, requires manual weighting that is problem-dependent and sensitive to factors such as loss magnitudes~\cite{chen2018gradnorm}. 
Importantly, such \emph{a priori} scalarization tends to bias optimization toward dominant objectives, e.g., prioritizing performance over fairness~\cite{pmlr-v119-martinez20a, zafar2017fairness} or favoring certain tasks over others~\cite{sener2018multi, yu2020gradient}.  
Consequently, it can lead to misleading interpretations of the underlying trade-offs and, more importantly, hinder the optimizer from reaching solutions close to the utopia point, thereby limiting overall performance. 
A principled solution to these issues is \emph{a posteriori} MOO, and in particular, Pareto front learning (PFL). However, to the best of our knowledge, \emph{there is currently no  PFL approach that is suitable for training deep neural networks at scale, while providing Pareto stationarity guarantees}. %

In this paper, we propose MOSEL (Multi-Objective Stackelberg Efficient Learning), the first framework for PFL in deep neural networks that recovers Pareto stationary solutions at scale. 
To that end, we exploit the modular structure of neural networks to decouple  representation learning from objective preference alignment, enabling efficient MOO training without sacrificing scalability. 
Specifically, we introduce a framework for PFL that reformulates the problem as a bilevel optimization, which we then cast into  a Stackelberg game, allowing the original problem to be solved efficiently with a single forward–backward pass of backpropagation. As a result, \emph{\ourmethod~ achieves Pareto stationary front learning with time and memory complexity comparable to conventional single-objective training}.

Beyond efficiency, we demonstrate that explicitly modeling multiple objectives allows \ourmethod\ to: (i) recover diverse and high-quality Pareto frontiers in strongly conflicting settings;  and, (ii)  converge to solutions closer to the utopia solution in weakly conflicting regimes such as  MTL, outperforming both single-objective baselines and specialized multi-task methods.
These findings suggest that \emph{efficient multi-objective learning can serve as a foundation for learning richer, more diverse representations, potentially improving robustness and generalization in deep neural networks.}

\section{Background and Related Work}

In this section, we introduce the key concepts from multi-objective and bilevel optimization that underpin our approach, and provide an overview of the related literature. 

\subsection{Multi-objective optimization.} 
In deep learning, MOO naturally arises when training a model with multiple loss functions. Let $\vec{\theta} \in \Theta \subset \mathbb{R}^n$ denote the parameters of a neural network, and consider $m$ non-negative, differentiable loss functions $\obj_j: \Theta \to \mathbb{R}_+$, which may correspond to task-specific losses in multi-task learning or stakeholder-defined losses (e.g., fairness objectives). The training problem can then be formulated as a MOO problem~\cite{ehrgott2005multicriteria, miettinen1999nonlinear}:
\begin{equation}
\label{eq:moo}
\min_{\vec{\theta} \in \Theta} ~ \obj(\vec{\theta}) = \min_{\vec{\theta} \in \Theta} ~ (\obj_1(\vec{\theta}), \ldots, \obj_m(\vec{\theta})).
\end{equation}

In MOO, the objectives are typically conflicting, meaning that no single parameter configuration simultaneously minimizes all losses. To formalize optimal trade-offs,~\cite{miettinen1999nonlinear} defines a dominance relation: a solution $\vec{\theta}_1 \in \Theta$ is said to be dominated by $\vec{\theta}_2 \in \Theta$, denoted by $\obj(\vec{\theta}_2) \prec \obj(\vec{\theta}_1)$, if $\vec{\theta}_2$ improves at least one objective without worsening any other. A point $\vec{\theta}^\star \in \Theta$ is called Pareto optimal (PO) if no other point in $\Theta$ dominates it. Given the non-convexity of neural network loss landscapes, we focus on local notions of optimality.

\textbf{Optimality. } The set of local Pareto optimal solutions~\cite{miettinen1999nonlinear} is defined as
\begin{equation}
    \mathcal{P} := \{\vec{\theta}^\star \in \Theta \mid \exists \epsilon > 0 : \forall \vec{\theta} \in \mathcal{N}_\epsilon(\vec{\theta}^\star) , \obj(\vec{\theta}) \nprec \obj(\vec{\theta}^\star) \}
\end{equation}
where $\mathcal{N}_\epsilon(\vec{\theta}^\star) = \{\vec{\theta}\in \Theta~|~\|\vec{\theta}-\vec{\theta}^\star\|<\epsilon\}$. The image set $\obj(\mathcal{P})$ is referred to as the local Pareto front (PF), which may consist of multiple or even infinitely many solutions. 

In differentiable settings, a point $\vec{\theta}^\star \in \Theta$ is Pareto stationary if there exists a convex combination of the gradients of the losses that vanishes. Formally, $\vec{\theta}^\star$ is Pareto stationary if
\begin{equation}
\label{eq:stationarity}
\exists\ \vec{r} \in \mathbb{R}^m, r_j \geq 0, \sum r_j = 1:~ \sum_{j=1}^m r_j\nabla \obj_j(\vec{\theta}^\star) = \mat{J}_\obj(\vec{\theta}^\star)^T\vec{r} = \vec{0},
\end{equation}
where $\mat{J}_\obj(\vec{\theta}^\star)$ denotes the Jacobian of the vector-valued objective. Pareto stationarity is a necessary condition for local Pareto optimality, and becomes sufficient under convexity assumptions~\cite{desideri2012multiple, fliege2000steepest}.

\textbf{Scalarization. } 
A natural approach to MOO is to reduce the vector-valued problem in Eq.~\eqref{eq:moo} to a single surrogate objective via \emph{scalarization}, which encodes a preference over the objectives through a vector $\vec{r}$~\cite{miettinen1999nonlinear}.  Two common scalarization functions that we will explore in this work include linear scalarization (LS)~\cite{miettinen1999nonlinear}, which translates to a weighted sum of the objectives as $\sum_{i=1}^m r_i\obj_i(\vec{\theta})$ and smooth Chebychev scalarization (CS)~\cite{lin2024smooth}, which is defined as $\frac{1}{\tau} \log \sum_{i=1}^{m} \exp\left( \tau\cdot r_i \mid \mathcal{L}_i(\vec{\theta}) - u_i^\star \mid \right)$ with smoothing parameter $\tau$ and \emph{utopia point} $\vec{u}^\star$, where each $u_i^\star$ is the ideal objective value of $\obj_i(\vec{\theta})$ attainable in isolation. 
In the presence of conflicting objectives, this point is generally infeasible. 
While linear scalarization is simple and efficient, it can only recover solutions on convex parts of the Pareto front, Chebychev scalarization on the other hand does not suffer from this limitation~\cite{miettinen1999nonlinear}. 

\textbf{Related works. } 
MOO methods can be categorized based on when preferences are specified~\cite{ liu2024stochastic, miettinen1999nonlinear}. \emph{No-preference} approaches resolve competing gradients without user input~\cite{desideri2012multiple, sener2018multi}. \emph{A priori} methods compute a single Pareto stationary solution for a fixed preference, typically via scalarization~\cite{lin2024smooth} or preference-aware gradient balancing~\cite{mahapatra2020multi}, but require retraining for new preferences. \emph{A posteriori} methods instead generate a set of trade-off solutions, e.g., by partitioning the objective space~\cite{lin2019pareto}, tracing Pareto sets~\cite{pmlr-v119-ma20a}, or sampling preferences for coverage~\cite{chen2022multi}. \emph{Interactive} methods iteratively refine solutions with user feedback~\cite{miettinen1999nonlinear}, but remain costly in deep learning, although approaches such as RLHF~\cite{christiano2017deep} and LLM alignment~\cite{ouyang2022training} can be seen as instances of this paradigm.

A key limitation of both \emph{a priori} and finite-solution \emph{a posteriori} methods is their inability to generalize to unseen preferences without retraining. \emph{Pareto Front Learning} (PFL) addresses this by learning a mapping from preferences to Pareto optimal solutions, aiming for (i) \emph{optimality} (i.e., solutions lying on the Pareto front) and (ii) \emph{diversity} (i.e., coverage of the full front)~\cite{coello2007evolutionary, hoang2022improving}. Pareto Hypernetworks~\cite{navon2021learning} achieve this but incur high computational and parameter costs due to per-objective gradients and auxiliary networks. COSMOS~\cite{ruchte2021scalable} improves scalability with a single training run and minimal overhead, but lacks guarantees of Pareto stationarity, as shown in Appendix~\ref{app:proofs}. 

In contrast to prior work, \emph{\ourmethod\ combines the strengths of \emph{a posteriori} methods and Pareto front learning while avoiding their key limitations, achieving Pareto stationary solutions without incurring significant computational or parameter overhead.}

\subsection{Bilevel optimization} A bilevel optimization problem introduces a hierarchy between two connected optimization problems:  %
\begin{equation}
\min_{\vec{\phi} \in \Phi} \quad  g(\vec{\phi}, \vec{\theta}) \quad \text{s.t.} \quad \vec{\theta} \in \operatorname*{arg\,min}_{\vec{\theta}' \in \Theta} f(\vec{\phi}, \vec{\theta}') 
\end{equation}
where $\Phi \subseteq \mathbb{R}^p$ and $\Theta \subseteq \mathbb{R}^q$ denote the feasible decision spaces for the upper and lower level problems; and $g$ and $f$ denote the upper and lower level objectives respectively~\cite{colson2007overview, dempe2002foundations}. The bilevel formulation provides a more natural way to describe an objective that inherently depends on the outcome of a second objective~\cite{sinha2017review}, allowing for more flexibility in the optimization process by enabling separate optimization for each level.

\textbf{Related works. } 
Previous works~\cite{chen2022multi, yazdani2024fair} have explored bilevel formulations of MOO. Rather than fixing preference vectors a priori,~\cite{chen2022multi} optimizes a distribution over a finite set of preferences in the upper level with respect to a quality metric (e.g., hypervolume) computed in the lower level. While this yields well-distributed solutions, the number of Pareto points remains fixed and limited by the upper-level design. 
In contrast,~\cite{yazdani2024fair} propose a bilevel formulation with accuracy as the upper-level and fairness as the lower level objective. However, scaling this approach beyond two objectives is non-trivial, as each additional objective requires introducing additional hierarchical levels.

In contrast to prior work, \emph{\ourmethod\  leverages bilevel optimization to efficiently learn Pareto stationary fronts without restricting the number of solutions or objectives.}  
\section{From \emph{a posteriori} MOO to bilevel optimization of neural networks}
\label{sec:bilevel}
In this section, we cast the \emph{a posteriori} MOO problem, or more specifically, the Pareto stationary front learning of deep neural networks into a bilevel optimization problem. 
To that end, we first assume a \textbf{scalarization function}, of the form $s: \mathbb{R}^m \times {[0,1]}^m \mapsto \mathbb{R}$, to aggregate the vector of loss functions (i.e., objectives) $\obj  = (\obj_1, \ldots, \obj_m)$ into a single scalar value, according to a given preference vector $\vec{r}\in [0,1]^m$.  In our case,  we consider linear scalarization $\sum_{i=1}^m r_i\obj_i(\vec{\theta})$~\cite{miettinen1999nonlinear} and smooth Tchebycheff scalarization $\frac{1}{\tau} \log \sum_{i=1}^{m} \exp\left( \tau \cdot r_i \mid \mathcal{L}_i(\vec{\theta}) - u_i^\star \mid \right)$~\cite{lin2024smooth}. %

Second, we leverage the modularity of neural networks to decouples the learning of optimal (i.e., Pareto stationary)  from diverse (i.e., different trade-offs among objectives) solutions, as illustrated in Figure~\ref{fig:method} (a).  
To this end, we partition the network into: %
(i)  a \textbf{representation learning network} $f_\theta$, which seeks representations steering the Pareto front toward the utopia point, effectively guiding solutions toward \emph{optimality}; and (ii) a \textbf{preference alignment network} $g_{\phi}$,  which focuses on maintaining \emph{diversity} of the solutions across the Pareto front.
By separating optimality and diversity within the neural network, \ourmethod\ mitigates adversarial interactions between competing goals, yielding a more stable optimization process. In Section~\ref{sec:MOSEL}, we formalize this intuition by showing that stochastic preference sampling implicitly regularizes the representation network toward shared features, exponentially suppressing objective-specific ones. Importantly, this decomposition comes at almost no additional memory or computational cost, as it simply requires the incorporation of a preference vector as input to the preference alignment network and partitioning the neural network computational graph during training. 

Figure~\ref{fig:method} (b) illustrates this architectural design using the ResNet-34 architecture~\cite{he2015deep}. The network is divided into two parts: a representation learning network, comprising the convolutional stem and the first three residual layers, and a preference alignment network consisting of the fourth residual layer followed by the fully connected head. We leverage the intermediate representation after the third residual layer, denoted by $\vec{z}$, as the conditioning point for preference injection. 
As highlighted in Figure~\ref{fig:method} (b), the only modification to the original architecture is a Feature-wise Linear Modulation (FiLM) layer~\cite{perez2017visual} inserted at the split point, which modulates $\vec{z}$ channel-wise based on the preference vector $\vec{r}$, steering the shared representation towards the desired trade-off while preserving its spatial structure.  The resulting conditioned representation $\vec{z}'$ is subsequently passed through the remaining network. This design introduces only negligible overhead relative to the base network.

\begin{figure}[t]
    \centering
    \includegraphics[width=\linewidth]{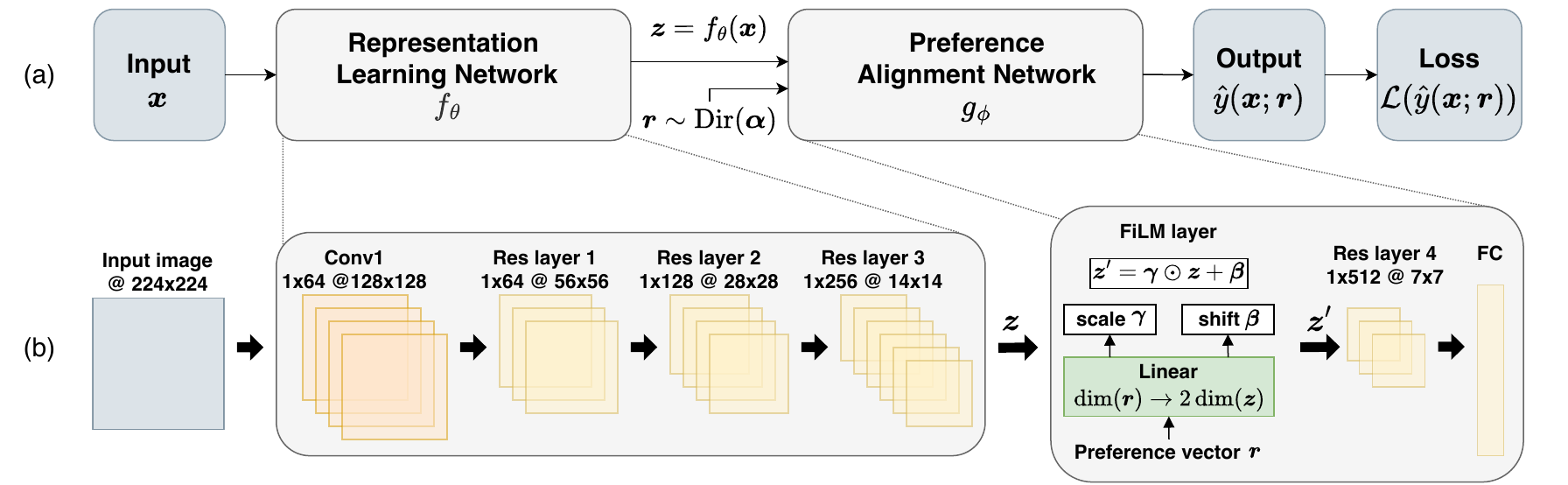}
    \caption{\textbf{Overview of \ourmethod.} (a) It partitions the neural network into a representation learning network, steering $\vec{z}$ toward Pareto optimal solutions, and a preference alignment network conditioning on $\vec{r}\sim \text{Dir}(\vec{\alpha})$ to cover diverse trade-offs. (b) Illustration on ResNet-34~\cite{he2015deep}: $\vec{z}$ is extracted at the network split point and modulated via a FiLM layer~\cite{perez2017visual}, which maps $\vec{r}$ to a scaling and shift factor per feature channel, steering the final predictions $\hat{y}(\vec{x};\vec{r})$ toward the desired objective trade-off.} 
    \label{fig:method} \vspace{-8pt}
\end{figure}
Third, we can now reformulate the original \emph{a posteriori} MOO problem from Eq.~\eqref{eq:moo} as the following \textbf{bilevel optimization problem}:
\begin{equation}
    \label{eq:bilevel}
    \begin{aligned}
        & \text{(leader) \hfill} & \underset{\vec{\phi}}{\text{arg min }}~  s(\obj(\vec{\theta}^\star, \vec{\phi}(\vec{r})); \vec{r}) \\
         &\text{(follower) \hfill} & s.t.~~ \vec{\theta}^\star \in \underset{\theta}{\text{arg min }}~ ~ s(\obj(\vec{\theta}, \vec{\phi}(\vec{r})); \vec{r})
    \end{aligned} 
    ~~~, \quad  \forall \vec{r}\sim Dir(\vec{\alpha}) , 
\end{equation}
where the preference vectors $\vec{r}$ are sampled from a Dirichlet distribution with hyperparameter $\vec{\alpha}$ and $s$ denotes the scalarization function.

In the context of Pareto front learning, the bilevel perspective provides a natural decomposition between optimality and diversity: the lower level ensures Pareto stationary solutions, while the upper level governs preference alignment and coverage of the front. We interpret the lower level optimization problem as the \emph{follower}, which optimizes the parameters of the representation network $\vec{\theta}$, and the upper level problem as the \emph{leader}, which optimizes the parameters of the preference alignment network $\vec{\phi}$.
We adopt this leader-follower perspective in Section~\ref{sec:MOSEL}, where we cast the bilevel problem as a Stackelberg game. By combining this interpretation with the modular structure of neural networks, we enable efficient single-pass training that recovers diverse, Pareto stationary solutions without restricting the number of objectives or attainable trade-offs. 

Before introducing the  Stackelberg interpretation, we first show that any solution to Eq.~\eqref{eq:bilevel} is a Pareto stationary solution of the original \emph{a posteriori} MOO problem in Eq.~\eqref{eq:moo}.
We assume that the scalarization function $s$ %
satisfies the following properties:

    (i) $s(\obj(\vec{\theta}, \vec{\phi}(\vec{r})); \vec{r})$ is continuously differentiable in $\vec{\theta}$ and $\vec{\phi}$\; and $s(\cdot,\vec{r})$ in $\obj$; and \\
    (ii) $s(\obj(\vec{\theta}, \vec{\phi}(\vec{r})); \vec{r})$ is strictly increasing in each objective $\obj_i$ with $\partial s/\partial \obj_i>0$ for all $i=1, \ldots, m$. 
Note that both linear scalarization and smooth Tchebycheff scalarization satisfy these conditions, since their partial derivatives w.r.t. each objective are strictly positive whenever $r_i > 0$.

\begin{theorem}[Pareto stationarity]
\label{thm:stationarity}
Let $s$ be a scalarization function satisfying properties~(i) and~(ii), and let $(\vec{\phi}^\star, \vec{\theta}^\star)$ be a solution to the bilevel problem in Eq.~\eqref{eq:bilevel}. Then $(\vec{\phi}^\star, \vec{\theta}^\star)$ is a Pareto stationary solution of the MOO problem in Eq.~\eqref{eq:moo}.

\proof
Let $\vec{\phi}^\star$ and $\vec{\theta}^\star$ denote solutions to the upper and lower level optimization problems, respectively. 
Define weights $\omega_i(\vec{r}) := \partial s/\partial \obj_i$ with partial derivatives evaluated at $\obj(\vec{\theta}^\star, \vec{\phi}^\star(\vec{r}))$ and normalize such that $\sum_{i=1}^{m} \omega_i(\vec{r}) = 1$. By property (ii), $\omega_i(\vec{r})>0$ so it lies in the probability simplex. 

Since $\vec{\theta}^\star$ is an unconstrained solution to the lower level problem, the first-order optimality condition together with property (i) gives 
\begin{equation}
    \label{eq:foc-theta}
    \nabla_{\vec{\theta}}\, s(\obj(\vec{\theta}^\star, \vec{\phi}^\star(\vec{r})); \vec{r}) = \left(\sum_{j=1}^m \frac{\partial s}{\partial \obj_j}\right) \sum_{i=1}^{m} \omega_i(\vec{r})\, \nabla_{\vec{\theta}} \obj_i(\vec{\theta}^\star, \vec{\phi}^\star(\vec{r})) = \vec{0}.
\end{equation}

The upper level objective depends on $\phi$ both directly and through the followers best response $\vec{\theta}^\star (\phi)$. Using the first-order optimality condition and Eq.~\eqref{eq:foc-theta} its total derivative is  
\begin{align}
    \label{eq:foc-phi}
    \begin{split}
        \frac{d~s(\obj(\vec{\theta}^\star, \vec{\phi}^\star(\vec{r})); \vec{r})}{d\vec{\phi}} & = \nabla_{\vec{\phi}} s(\obj(\vec{\theta}^\star, \vec{\phi}^\star(\vec{r})); \vec{r}) + \left(\frac{\partial\vec{\theta}^\star}{\partial\phi}\right)^\top \nabla_{\vec{\theta}} s(\obj(\vec{\theta}^\star, \vec{\phi}^\star(\vec{r})); \vec{r}) \\
        & = \left(\sum_{j=1}^m \frac{\partial s}{\partial \obj_j}\right) \sum_{i=1}^{m} \omega_i(\vec{r})\, \nabla_{\vec{\phi}} \obj_i(\vec{\theta}^\star, \vec{\phi}^\star(\vec{r})) = \vec{0}.
    \end{split}
\end{align}
Since $\sum_j \partial s/\partial \obj_i$ is strictly positive, we can divide it out of both Eq.~\eqref{eq:foc-theta} and~\eqref{eq:foc-phi}. Since both equations contain the same weights we can combine them to expresses $\vec{0}$ as a convex combination of the per-objective gradients at $(\vec{\phi}^\star, \vec{\theta}^\star)$, which is precisely the Pareto stationarity condition in Eq.~\eqref{eq:stationarity}. Thus, any solution $(\vec{\phi}^\star, \vec{\theta}^\star)$ to the bilevel optimization problem in Eq.~\eqref{eq:bilevel} is Pareto stationary. ~ $\blacksquare$
\end{theorem}

\section{MOSEL: Multi-Objective Stackelberg Efficient Learning}\label{sec:MOSEL}

In this section, we introduce Multi-Objective Stackelberg Efficient Learning (\ourmethod), a framework for \emph{a posteriori} multi-objective optimization in deep neural networks that recovers a full front of Pareto stationary solutions at the computational cost of standard single-objective training.
To that end, we build on recent works that leverage Stackelberg game theory for deep learning optimization---such as gradient homogenization in multi-task learning~\cite{javaloy2022rotograd} and fairness–accuracy trade-offs~\cite{yazdani2024fair}---to cast the bilevel optimization problem in Eq.~\eqref{eq:bilevel} to a Stackelberg game. 
Following the notion of a differentiable Stackelberg equilibrium, introduced in~\cite{pmlr-v119-fiez20a}, we train both the representation learning and preference alignment networks via simultaneous gradient updates---in other words, via a single forward-backward pass.

A general \textbf{Stackelberg game} is a two-player game consisting of a leader and follower, where the leader plays first and the follower reacts optimally to this move, representing the hierarchical structure of bilevel optimization problems. 
As stated by~\citet{pmlr-v119-fiez20a}, given the objectives of leader and follower $g, f : \Phi \times \Theta \to \mathbb{R}$, a strategy $(\vec{\phi}^\star, \vec{\theta}^\star)$ is a Stackelberg equilibrium, if it minimizes %
\begin{equation}
\operatorname*{arg\,min}_{\vec{\phi} \in \Phi} \quad  g(\vec{\phi}, \vec{\theta}) \quad \text{s.t.} \quad \vec{\theta} \in \operatorname*{arg\,min}_{\vec{\theta}' \in \Theta} f(\vec{\phi}, \vec{\theta}'). 
\end{equation}

As we operate in non-convex settings, we consider the corresponding local formulation, where the strategy $\vec{\phi}^\star \in U_\phi$ is a local Stackelberg solution for the leader if, for all $\vec{\phi} \in U_\phi$,
\begin{equation}
    \sup_{\vec{\theta} \in R_{U_\theta}(\vec{\phi}^\star)} g(\vec{\phi}^\star, \vec{\theta}) 
    \;\leq\; 
    \sup_{\vec{\theta} \in R_{U_\theta}(\vec{\phi})} g(\vec{\phi}, \vec{\theta}),
\end{equation}
where $ R_{U_\theta}(\vec{\phi}) = \bigl\{\vec{\theta} \in U_\theta \mid f(\vec{\phi}, \vec{\theta}) \leq f(\vec{\phi}, \vec{\theta}'),\, \forall \vec{\theta}' \in U_\theta\bigr\}$ is the follower's local best-response. Moreover, for any $\vec{\theta}^\star \in R_{U_\theta}(\vec{\phi}^\star)$, $(\vec{\phi}^\star, \vec{\theta}^\star)$  is a local Stackelberg equilibrium on $U_\phi \times U_\theta$.
However, the exact computation of $R_{U_\theta}(\vec{\phi})$ is intractable in neural networks. 

To overcome this limitation,~\citet{pmlr-v119-fiez20a} introduce the notion of \textbf{differentiable Stackelberg equilibrium}, where $(\vec{\phi}^\star, \vec{\theta}^\star) = (\vec{\phi}^\star, b(\vec{\phi}^\star))$ is a differentiable Stackelberg equilibrium, if  
\begin{equation}
    \nabla_{\vec{\phi}}\, g(\vec{\phi}^\star, b(\vec{\phi}^\star)) = 0 \quad
    \text{and} \quad
    \nabla^2_{\vec{\phi}}\, g(\vec{\phi}^\star, b(\vec{\phi}^\star)) \succ 0,
\end{equation}
where $\vec{\theta}^\star = b(\vec{\phi}^\star)$ is the follower's best response that is implicitly defined by $\nabla_{\vec{\theta}} f(\vec{\phi}^\star, \vec{\theta}^\star) = 0$.
The concepts of local Stackelberg equilibrium and differentiable Stackelberg equilibrium coincide except on a set of measure zero~\cite{pmlr-v119-fiez20a}, making the differentiable Stackelberg equilibrium a natural target for gradient-based optimization. 
~\citet{pmlr-v119-fiez20a} further show, that simultaneous gradient updates with a timescale separation, that is, a smaller learning rate for the leader than for the follower, is sufficient to ensure local convergence to a differentiable Stackelberg equilibrium. 
Intuitively, the faster follower updates keep $\vec{\theta}$ close to its best response at every leader step, approximating the hierarchical structure of the Stackelberg game without requiring exact best-response computation.
\begin{algorithm}[t]
\caption{\ourmethod~ algorithm}\label{alg:alg}
\begin{algorithmic}
    \Require Dataset $\mathcal{D} = (x, y)$,  Dirichlet hyperparameter $\vec{\alpha}$,\\ ~~~~~~~~~learning rates $\eta_\phi$ and   $\eta_\theta$ \Comment{$\eta_\phi < \eta_\theta$}     
\For{$e$ epochs}
    \For{every batch $(x, y) \sim p_\mathcal{D}$}
        \State Sample $\vec{r} \sim \mathrm{Dir}(\vec{\alpha})$
            \Comment{1 preference vector per batch}\\
         \quad\quad (\textit{Forward Pass})
        \State $\hat{y} \leftarrow g_{\vec{\phi}}(f_{\vec{\theta}}(x),\, \vec{r})$ \Comment{Compute neural network outputs}
        \State $\ell := s(\obj(\vec{\theta}, \vec{\phi}(\vec{r})); \vec{r})$ \Comment{Evaluate losses on $(\hat{y},y)$ and scalarize}\\
         \quad\quad (\textit{Backward Pass})
        \State  $\vec{\phi} \leftarrow \vec{\phi} - \eta_\phi\, \nabla_{\vec{\phi}} \ell$ \Comment{Update preference alignment network}
        \State  $\vec{\theta} \leftarrow \vec{\theta} - \eta_\theta\, \nabla_{\vec{\theta}} \ell$ \Comment{Update representation learning network}
    \EndFor
\EndFor
 
\State \Return $\vec{\phi}, \vec{\theta}$
\end{algorithmic}
\end{algorithm}%
These guarantees directly connect to the Pareto stationarity condition in Eq.~\eqref{eq:stationarity}, as any differentiable Stackelberg equilibrium $(\vec{\phi}^\star, \vec{\theta}^\star)$ satisfy their first-order stationarity conditions, which coincide with the condition in Eq.~\eqref{eq:stationarity}. 
Thus, convergence to a differentiable Stackelberg equilibrium implies convergence to a Pareto stationary solution, providing the theoretical foundation for the \textbf{\ourmethod\ algorithm} in Alg.~\ref{alg:alg}. 
We highlight that {the output of Alg.~\ref{alg:alg} correspond to local optima of the neural network parameters $\vec{\phi}$ and $\vec{\theta}$, which in turn yield a Pareto stationary solution of the bilevel optimization problem in Eq.~\eqref{eq:bilevel}, and consequently of the original MOO problem in Eq.~\eqref{eq:moo}}.

Beyond Pareto stationarity, the stochastic preference sampling in Alg.~\ref{alg:alg} implicitly regularizes the representation network toward shared features. Under the power-iteration characterization of early-stage gradient descent~\citep{zhang2026neural}, 
we show in Theorem~\ref{thm:variance_penalty} in App.~\ref{app:regulariaztion} that objective-specific features are exponentially suppressed relative to shared ones, with strength controlled by the Dirichlet concentration $\vec\alpha$. The following proposition summarizes the main results.
\begin{proposition}[Implicit regularization toward shared features]
Consider the bilevel problem in Eq.~\eqref{eq:bilevel} with two objectives $A$ and $B$, solved via the Stackelberg formulation of Alg.~\ref{alg:alg}. Let $\mathcal{X}_A, \mathcal{X}_B \subseteq \mathcal{X}$ denote the feature directions in $f_\theta$ relevant to each objective and $\mathcal{X}_{AB} = \mathcal{X}_A \cap \mathcal{X}_B$ the shared features. In the weakly conflicting regime, i.e., when $\mathcal{X}_{AB} \neq \emptyset$, let $\rho_{k,m}$ denote the relative learning ratio, i.e., the fraction of representational energy allocated to feature $k$ over feature $m$, under the stochastic preference sampling of Alg.~\ref{alg:alg}. Then:
\begin{equation*}
    \forall~k \in \mathcal{X}_{AB},~m \in \mathcal{X}\setminus \mathcal{X}_{AB}~:~ \exists~\vec{\alpha} : \rho_{k,m} > 1.
\end{equation*}
\end{proposition}
Thus, there exists a Dirichlet parameter $\vec{\alpha}$, such that the network preferentially learns shared features over objective-specific ones, even when the latter has a larger average importance as shown in Remark~\ref{rem:domination}.
The full proof, assumptions, and generalization to $K$ objectives can be found in App.~\ref{app:regulariaztion}.

\textbf{Computational time and memory complexity. }
Note that \ourmethod\ shares the same asymptotic complexity as COSMOS~\cite{ruchte2021scalable}, with $\mathcal{O}(e \cdot m \cdot p)$ time and $\mathcal{O}(e \cdot p)$ memory, where $e$ denotes the number of epochs and $p$ the number of network parameters. This efficiency arises because both methods perform a single forward and backward pass over $m$ objectives per gradient step. Importantly, the proposed split architecture introduces no additional overhead: the timescale separation approximation reduces the nested second-order derivatives of exact bilevel optimization to standard gradient descent, thereby achieving Pareto stationarity at no extra computational cost.
\vspace{-4pt}

\section{Experiments}
\label{sec:exp}

\subsection{Strongly conflicting setting: Fairness-Performance trade-off}
Fairness in machine learning is inherently a multi-objective problem, as metrics such as Difference in Equality of Opportunity (DEO) directly conflict with predictive performance~\cite{padh2021addressing, wei2022fairness}. Existing approaches typically treat fairness as a constraint or penalty, reducing the problem to a single trade-off and limiting the information available to decision makers~\cite{liu2022accuracy}. This is particularly problematic in high-stakes domains such as criminal justice, credit scoring, and healthcare~\cite{pmlr-v119-martinez20a, wei2022fairness}. 
We leverage this use case in the following experiments to empirically validate that \ourmethod~recovers diverse and high-quality Pareto frontiers in strongly conflicting settings.

\textbf{Datasets. }
We evaluate on Adult~\cite{adult_dataset}, COMPAS~\cite{angwin2016compas}, and Credit~\cite{hofmann1994german}, three widely used tabular fairness benchmarks with \texttt{'sex'} as the sensitive attribute, and on the image dataset FairFace~\cite{karkkainenfairface}, predicting \texttt{'gender'} under sensitive attribute \texttt{'race'}. Training objectives are binary cross-entropy (BCE) as a proxy for accuracy and a differentiable hyperbolic tangent relaxation of DEO~\cite{padh2021addressing} as a fairness surrogate; at test time we report misclassification error and DEO.

\textbf{Baselines and metrics. }
We compare \ourmethod~against COSMOS~\cite{ruchte2021scalable}, evaluating both on 25 uniformly distributed test rays. We additionally compare against the post-processing baseline of~\citet{hardt2016equality}, which operates on an unconstrained classifier (using the same neural network architecture) by sweeping a shared target TPR across per-group ROC convex hulls to derive group-specific decision thresholds, tracing out the Pareto front. Further details on hyperparameters, architectures, and training are provided in Appendix~\ref{app:exp}. 
We report the relative hypervolume~\cite{zitzler2007hypervolume}, defined as the achieved hypervolume divided by the maximum possible hypervolume between the utopia point (ideal, generally unattainable) and nadir point (worst per-objective value among observed solutions) as shown in Appendix~\ref{app:hv}; alongside the number of non-dominated solutions per method, where dominated solutions are removed in post-processing. 

\textbf{Results. }
Table~\ref{tab:fairness_app} summarizes the results on both training and test objectives. \ourmethod-LS and \ourmethod-CS consistently match or outperform COSMOS across datasets in terms of both hypervolume and the number of Pareto optimal solutions, providing strong empirical evidence that \ourmethod\ recovers diverse and high-quality Pareto frontiers in strongly conflicting settings. This advantage is particularly pronounced in Figure~\ref{fig:fairness}, where on FairFace, COSMOS collapses to a near-degenerate front, while both MOSEL variants maintain high hypervolume and a substantially larger set of Pareto optimal solutions.
The post-hoc fairness baseline~\cite{hardt2016equality} remains competitive in hypervolume in some cases (e.g., Default and FairFace test error), but generally yields fewer non-dominated solutions. Moreover, on datasets such as Adult, its solutions are concentrated in a narrow region of low unfairness but high error, limiting its practical utility. %
\begin{table}[h]
    \centering
    \caption{\textbf{Fairness-Performance Trade-offs.} Relative hypervolume and number of Pareto optimal solutions on the fairness benchmarks, computed on training loss and test error. Mean~$\pm$~std over five runs; best entries per column shaded.}
    \label{tab:fairness_app}
    \resizebox{0.9\textwidth}{!}{%
    \begin{tabular}{c  c  cc  cc }
    \toprule
       Dataset & Method & HV Loss $\uparrow$ & HV Error $\uparrow$ & \# PO sol. (Loss)  $\uparrow$ & \# PO sol. (Error)  $\uparrow$ \\
\midrule
\multirow{4}{*}{Adult} & MOSEL-LS & \cellcolor{gray!30}0.939 $\pm$ 0.007 & \cellcolor{gray!30}0.849 $\pm$ 0.012 & 22.8 $\pm$ 1.2 & \cellcolor{gray!30}21.2 $\pm$ 1.5 \\
 & MOSEL-CS & 0.819 $\pm$ 0.017 & 0.776 $\pm$ 0.015 & \cellcolor{gray!30}23.8 $\pm$ 0.7 & \cellcolor{gray!30}22.0 $\pm$ 1.4 \\
 & COSMOS & 0.926 $\pm$ 0.014 & 0.830 $\pm$ 0.010 & 20.2 $\pm$ 0.7 & 17.4 $\pm$ 1.9 \\
 & Classifier & 0.906 $\pm$ 0.025 & 0.618 $\pm$ 0.016 & 8.0 $\pm$ 0.6 & 6.2 $\pm$ 0.7 \\
\midrule
\multirow{4}{*}{COMPAS} & MOSEL-LS & 0.920 $\pm$ 0.019 & 0.568 $\pm$ 0.008 & \cellcolor{gray!30}15.6 $\pm$ 1.4 & \cellcolor{gray!30}14.6 $\pm$ 1.4 \\
 & MOSEL-CS & \cellcolor{gray!30}0.930 $\pm$ 0.007 & 0.560 $\pm$ 0.005 & \cellcolor{gray!30}15.2 $\pm$ 1.6 & \cellcolor{gray!30}14.2 $\pm$ 1.6 \\
 & COSMOS & \cellcolor{gray!30}0.928 $\pm$ 0.002 & 0.567 $\pm$ 0.019 & \cellcolor{gray!30}14.6 $\pm$ 1.0 & \cellcolor{gray!30}13.6 $\pm$ 1.0 \\
 & Classifier & 0.881 $\pm$ 0.026 & \cellcolor{gray!30}0.645 $\pm$ 0.037 & 9.2 $\pm$ 1.6 & 6.2 $\pm$ 1.5 \\
\midrule
\multirow{4}{*}{Default} & MOSEL-LS & \cellcolor{gray!30}0.951 $\pm$ 0.011 & \cellcolor{gray!30}0.922 $\pm$ 0.014 & \cellcolor{gray!30}17.0 $\pm$ 1.7 & \cellcolor{gray!30}15.4 $\pm$ 0.8 \\
 & MOSEL-CS & 0.902 $\pm$ 0.029 & 0.897 $\pm$ 0.024 & 13.6 $\pm$ 2.1 & 10.8 $\pm$ 2.3 \\
 & COSMOS & 0.938 $\pm$ 0.009 & 0.905 $\pm$ 0.016 & 14.6 $\pm$ 3.9 & 11.8 $\pm$ 2.5 \\
 & Classifier & \cellcolor{gray!30}0.941 $\pm$ 0.023 & \cellcolor{gray!30}0.925 $\pm$ 0.016 & 4.8 $\pm$ 1.3 & 4.0 $\pm$ 1.1 \\
\midrule
\multirow{4}{*}{FairFace} & MOSEL-LS & 0.987 $\pm$ 0.013 & \cellcolor{gray!30}0.701 $\pm$ 0.047 & 13.6 $\pm$ 1.0 & \cellcolor{gray!30}10.6 $\pm$ 2.3 \\
 & MOSEL-CS & \cellcolor{gray!30}0.998 $\pm$ 0.003 & \cellcolor{gray!30}0.734 $\pm$ 0.042 & 13.2 $\pm$ 2.3 & \cellcolor{gray!30}11.4 $\pm$ 2.4 \\
 & COSMOS & 0.193 $\pm$ 0.038 & 0.176 $\pm$ 0.037 & 3.2 $\pm$ 1.0 & 2.0 $\pm$ 0.9 \\
 & Classifier & 0.808 $\pm$ 0.009 & \cellcolor{gray!30}0.708 $\pm$ 0.030 & \cellcolor{gray!30}16.4 $\pm$ 1.4 & \cellcolor{gray!30}12.2 $\pm$ 1.6 \\
\bottomrule
     \end{tabular}}
\end{table}

\begin{figure}[ht]
    \centering
    \includegraphics[width=\linewidth]{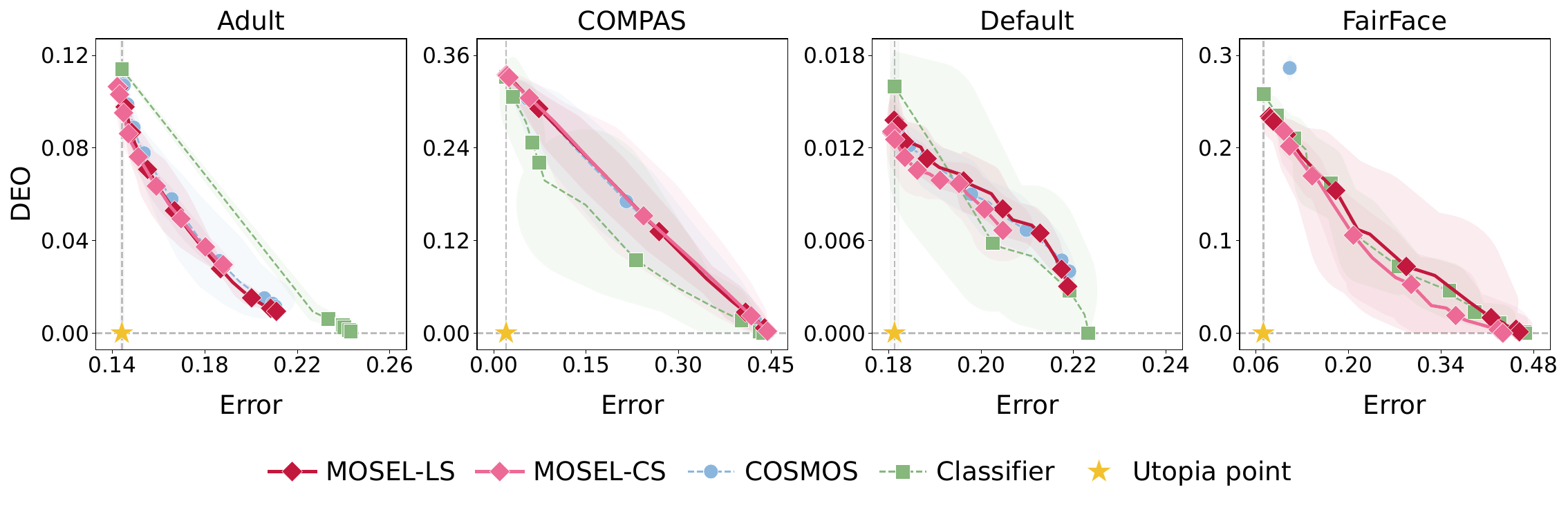}
    \caption{\textbf{Fairness-Performance Trade-offs.} Test-time Pareto fronts, plotting classification error against fairness violation. The grey dotted lines indicate zero unfairness and the best accuracy achievable by an unconstrained classifier; their intersection defines the utopia point—an ideal, yet unattainable solution toward which all methods aim to converge. }%
    \label{fig:fairness} \vspace{-10pt}
\end{figure}

\subsection{Weakly conflicting setting: Multi-Task Learning}

In multi-task learning, joint training can improve task performance via positive transfer, but task-specific losses may also induce conflicting gradients, degrading performance~\cite{sener2018multi}. However, such conflicts are not guaranteed as they depend on diverse factors like the data, training dynamics, etc., placing MTL in a weakly conflicting regime. As a result,   methods designed to resolve gradient conflicts—such as GradNorm~\cite{chen2018gradnorm}, FAMO~\cite{liu2023famo}, NashMTL~\cite{navon2022multi}, and PCGrad~\cite{yu2020gradient} —can even harm performance~\cite{liu2021conflict, shi2023recon}. \\
We use the MTL setting to demonstrate that \ourmethod\ adapts to both regimes without prior knowledge of the conflict structure. First, in a simple two-task setting, we show that \ourmethod\ converges to solutions closer to the utopia point by effectively leveraging positive transfer, outperforming both standard single-objective training and specialized multi-task learning methods without additional computational cost.
In a second experiment, we demonstrate that \ourmethod\ efficiently enables positive transfer in settings with many tasks (objectives). %

\begin{figure}[t]
    \centering
    \includegraphics[width=\linewidth]{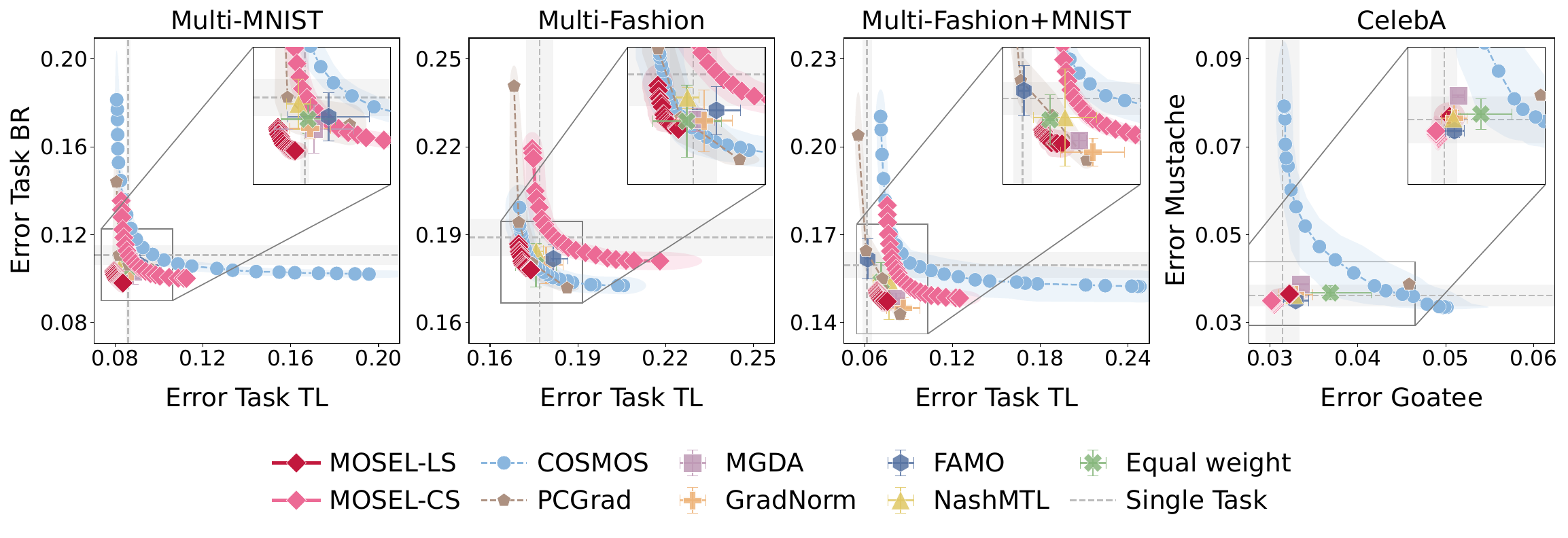}
    \caption{\textbf{Two-task setting.} Test-time Pareto fronts showing the trade-offs between task-specific errors. Grey dotted lines indicate the best achievable single-task performance on each task; their intersection marks the single-task baseline. In this weakly conflicting setting, positive transfer allows methods to surpass this baseline, with \ourmethod~achieving the largest improvements.}
    \label{fig:mtl}\vspace{-5pt}
\end{figure}

\begin{figure}
    \centering
    \includegraphics[width=\linewidth]{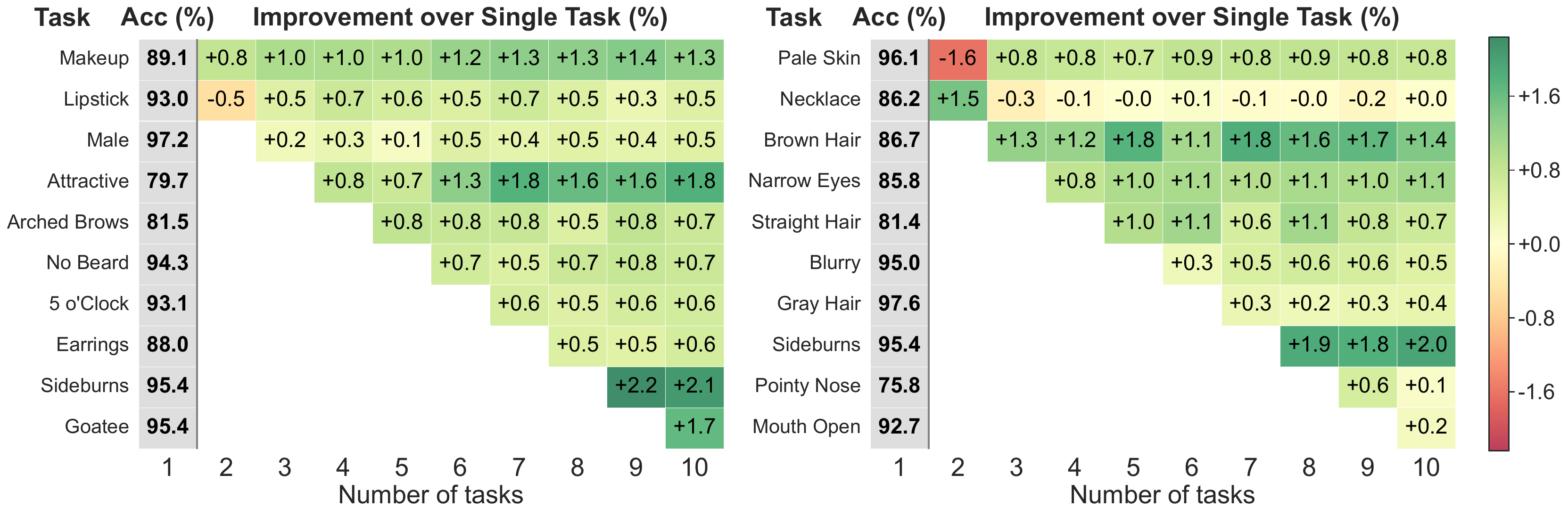}
    \caption{\textbf{Scalability to number of objectives on CelebA}. Per-task accuracy improvement over the single-task baseline (gray column) for the most correlated (left) and least correlated (right) task sets, as the number of jointly optimized tasks increases. Results are averaged over 5 seeds and evaluated on the Pareto stationary solution closest to the zero vector.} 
    \label{fig:celeba} \vspace{-15pt}
\end{figure}

\textbf{Datasets. }
We evaluate on three synthetic and one real-world benchmark. Multi-MNIST~\cite{sabour2017dynamic}, Multi-Fashion and Multi-Fashion+MNIST~\cite{lin2019pareto} are synthetic multi-task benchmarks that are composed of two digits or clothing items based on the MNIST~\cite{lecun1999object} and Fashion-MNIST~\cite{xiao2017fashion} dataset,where the goal is to simultaneously classify both objects (top-left and bottom-right), using per-task cross-entropy as objectives.  As real-world benchmark we consider the CelebA dataset~\cite{liu2015faceattributes}, containing over 200,000 images of celebrities, annotated with 40 binary attributes. %

\textbf{Baselines and metrics. }
We compare against COSMOS~\cite{ruchte2021scalable}, MGDA~\cite{sener2018multi}, PCGrad~\cite{yu2020gradient}, GradNorm~\cite{chen2018gradnorm}, FAMO~\cite{liu2023famo}, NashMTL~\cite{navon2022multi}, equal weighting and single task performances. While standard PCGrad uses equal weighting, we weigh the objectives based on these preferences to obtain diverse solutions, retraining and evaluating on 5 preferences. As other baselines do not support preference conditioning, we report the single obtained solution. Details on hyperparameters, architectures, and training configurations for all experiments are provided in Appendix~\ref{app:exp}. We train all methods using  (binary) cross-entropy; and report misclassification error at test time.

\textbf{Results on two-task setting. } As shown in Figure~\ref{fig:mtl}, \ourmethod~recovers solutions closest to the utopia point on all synthetic benchmarks, consistently outperforming COSMOS, most of the MTL methods, and the equal weight baseline. Notably, even in these weakly conflicting settings, where gradient-based MTL methods are specifically designed to help, 
\ourmethod~matches or surpasses their performance, demonstrating its benefits regardless of the degree of gradient conflict. 
\begin{wraptable}[10]{l}{0.5\linewidth}
\vspace{-7pt}
    \label{tab:comp_celeba}
    \resizebox{\linewidth}{!}{%
    \begin{tabular}{ c  c  c  c }     
    \toprule
     Method & $\Delta$ Param. & Time ($\times$) & CPU ($\times$)  \\
     \midrule
    Equal weight & 17,495,136 & $248.4 \pm 45.7$\,min & $11.15 \pm 0.25$\,MB \\
     \cmidrule(lr){1-4}
        MOSEL-LS    & +672   & $\times1.0 \pm 0.1$ & $\times0.97 \pm 0.02$ \\
        MOSEL-CS    & +672   & $\times1.0 \pm 0.2$ & $\times0.98 \pm 0.02$ \\
        COSMOS      & +1,072 & $\times1.1 \pm 0.1$ & $\times0.99 \pm 0.02$ \\
        PCGrad      & +0     & $\times9.8 \pm 2.2$ & $\times1.06 \pm 0.01$ \\
        MGDA        & +0     & $\times3.0 \pm 0.5$ & $\times0.99 \pm 0.03$ \\
        GradNorm    & +0     & $\times3.0 \pm 0.0$ & $\times0.95 \pm 0.02$ \\
        FAMO        & +0     & $\times1.2 \pm 0.1$ & $\times0.96 \pm 0.03$ \\
        NashMTL     & +0     & $\times2.4 \pm 0.2$ & $\times1.28 \pm 0.01$ \\
        Single Task & +0     & $\times2.0 \pm 0.3$ & $\times1.01 \pm 0.03$\\
    \bottomrule
    \end{tabular}}
\end{wraptable} 
Importantly, beyond solution quality, \ourmethod\ also excels in efficiency. The table on the left shows the computational overhead on CelebA relative to equal weight training (top); \# Param.\ are absolute counts; time and CPU are multiplicative factors; and,  means and standard deviations are computed over five runs. 
As reported, \ourmethod\ introduces negligible computational overhead, adding only 672 parameters and matching the training time of equal weight training, while PCGrad incurs a nearly $10\times$ slowdown due to per-preference retraining, and single-task training scales linearly with the number of tasks. Together, these results confirm that \ourmethod~recovers diverse and high-quality Pareto fronts at the cost of standard single-objective training.

 \textbf{Results on scalability to number of objectives. } 
To further assess scalability, we conduct an ablation on CelebA comparing correlated and uncorrelated task sets as the number of jointly optimized objectives grows (Figure~\ref{fig:celeba}), with tasks added incrementally and performance measured relative to single-task baselines. Task correlation is measured by the co-occurrence frequency of attributes in images. For \emph{correlated tasks} (left panel), shared representations enable strong positive transfer, with accuracy improvements over the single-task baseline growing consistently as more objectives are added. For \emph{uncorrelated tasks} (right panel), positive transfer is weaker, yet \ourmethod\ still outperforms single-task training on the majority of objectives. %
The occasional degradations can be attributed to two factors: tasks with high single-task accuracy have little room for improvement, and by choosing the Pareto optimal solution closest to the zero vector for evaluation, high-loss tasks dominate the distance computation, while low-loss tasks (high single-task accuracy) contribute minimally.
Nonetheless, the overall performance gains come at no additional cost: computational overhead remains negligible as the number of objectives grows, Figure~\ref{fig:runtime} shows that runtime stays approximately constant, with only minor fluctuations; while memory increases only marginally, highlighting the scalability of \ourmethod.  %

\section{Conclusion} \label{sec:conclusion}
We propose an efficient \emph{a posteriori} multi-objective framework for deep neural networks that achieves Pareto stationary solutions at the cost of standard single-objective training via a bilevel Stackelberg formulation enabling single-pass updates. 
Empirically, it recovers diverse Pareto frontiers in strongly conflicting settings and converges closer to the utopia point in weakly conflicting regimes such as multi-task learning, demonstrating that efficient multi-objective learning scales and improves performance by learning more diverse representations to capture all objectives.

\textbf{Limitations and future work. }
Despite its advantages, our approach has several limitations that open avenues for future research. 
First, our results suggest that efficient multi-objective learning can yield richer, more robust representations, especially in weakly conflicting regimes, motivating further study on %
transfer~\cite{pan2009survey} and alignment~\cite{he2025pareto}.
Moreover, our fairness–performance experiments show that our method achieves better trade-offs while providing diverse solutions for informed deployment; extending this to other high-stakes domains such as clinical risk prediction~\cite{martinez2019fairness}, recidivism assessment~\cite{liu2022accuracy} or credit scoring~\cite{kozodoi2022fairness} is an important direction for future work. %
Second, the method relies on a modular decomposition of neural networks, which may not be straightforward to design or optimal for all architectures; exploring adaptive or learned decompositions is therefore another important avenue.
Third, while the computational overhead of \ourmethod\ for recovering a Pareto stationary front is minimal, deploying the model requires selecting a preference vector—and thus a specific trade-off. This decision becomes increasingly challenging as the number of objectives grows and remains an open problem~\cite{javaloy2025copa}; addressing it via human-in-the-loop methods for multi-objective evaluation and selection is a key direction for future work. %

Overall, our results highlight the promise of our approach and, together with these challenges, motivate a broader research agenda at the intersection of scalable multi-objective learning, adaptive model design, and decision-aware deployment of deep neural models.

\begin{ack} %
We would like to thank Nektarios Kalampalikis for the insightful discussions and continuous support throughout this project.
This work has been supported by the project “Society-Aware Machine Learning: The paradigm shift demanded by society to trust machine learning,” funded by the European Union and led by IV (ERC-2021-STG, SAML, 101040177). Views and opinions expressed are, however, those of the author(s) only and do not necessarily reflect those of the aforementioned funding agencies. Neither of the aforementioned parties can be held responsible for them.
\end{ack}

\newpage
\bibliographystyle{plainnat}
\bibliography{references}

\newpage
\appendix

\section{Proofs}
\label{app:proofs}
\subsection{Non-stationarity of COSMOS solutions}
\begin{theorem}
Let $G$ be the objective function of the COSMOS algorithm, defined as:
\begin{equation}
    G = \vec{r}^T \obj - \lambda \frac{\vec{r}^T \obj}{\|\vec{r}\| \|\obj\|}
\end{equation}
where $\obj \in \mathbb{R}^m_{\geq 0}$ is the loss vector, $\vec{r} \in \mathbb{R}^m_{\geq 0}$ is the target preference vector on the probability simplex, i.e., $\sum r_i = 1$, and $\lambda > 0$ is a scalar hyperparameter. Assume that at any stationary point, the objective gradients exhibit non-degenerate conflict such that $\text{rank}(J_\obj) = m - 1$. As $r_i \to 1$, if the hyperparameter violates the bound $\lambda \leq \frac{\|\obj\|^2}{\|\obj\| - \obj_i}$, %
the stationary points of $G$ are strictly not Pareto stationary.
\end{theorem}

\begin{proof}
To evaluate the Pareto stationarity of $G$, we first find its gradient with respect to parameters $\theta$:
\begin{equation}
    \nabla_\theta G = J_\obj^T \vec{r} - \lambda J_\obj^T \left[ \frac{1}{\|\vec{r}\| \|\obj\|} \left( \vec{r} - \frac{\vec{r}^T \obj}{\|\obj\|^2} \obj \right) \right]
\end{equation}

We can rewrite this by factoring out the Jacobian transpose $J_\obj^T$ to isolate the effective weighting vector $\vec{r}'$:
\begin{equation}
    \nabla_\theta G = J_\obj^T \underbrace{\left[ \left( 1 - \frac{\lambda}{\|\vec{r}\| \|\obj\|} \right) \vec{r} + \left( \frac{\lambda (\vec{r}^T \obj)}{\|\vec{r}\| \|\obj\|^3} \right) \obj \right]}_{\vec{r}'}
\end{equation}

At a stationary point of the COSMOS algorithm, $\nabla_\theta G = 0$, which implies $J_\obj^T \vec{r}' = 0$. By our non-degenerate conflict assumption, the left null space of $J_\obj$ is 1-dimensional. Therefore, any valid Pareto preference vector $\vec{r}$ satisfying $J_\obj^T \vec{r} = 0$ must be a scalar multiple of $\vec{r}'$, i.e., $\vec{r} = c \cdot\vec{r}'$ for some $c \neq 0$. Consequently, for the point to be Pareto stationary, $\vec{r}'$ must be either strictly non-negative or strictly non-positive. Since $\vec{r}^T \obj > 0$ and $\obj \geq 0$, the vector $\vec{r}'$ cannot be strictly non-positive, and therefore Pareto stationarity strictly requires $\vec{r}' \geq 0$.

Enforcing $\vec{r}' \geq 0$ yields:
\begin{equation}
    \obj \geq - \left( \|\vec{r}\| \|\obj\| - \lambda \right)\left( \frac{ \|\obj\|^2}{\lambda (\vec{r}^T \obj)} \right) \vec{r}~.
\end{equation}

Rearranging the terms, we get:
\begin{equation}
    \obj \geq \left( \frac{\lambda - \|\vec{r}\| \|\obj\|}{\lambda} \right)\left( \frac{ \|\obj\|^2}{\vec{r}^T \obj} \right) \vec{r}~.
\end{equation}

By the Cauchy--Schwarz inequality, $\vec{r}^T \obj \leq \|\vec{r}\| \|\obj\|$. Because $\vec{r}^T \obj > 0$, we invert this to establish a strict lower bound for the fractional term:
\begin{equation}
    \frac{1}{\vec{r}^T \obj} \geq \frac{1}{\|\vec{r}\| \|\obj\|} \implies \frac{\|\obj\|^2}{\vec{r}^T \obj} \geq \frac{\|\obj\|}{\|\vec{r}\|}~.
\end{equation}

Substituting this lower bound back into our inequality, we establish the necessary condition for Pareto stationarity for any component $i$:
\begin{equation}
    \obj_i \geq \left( 1 - \frac{\|\vec{r}\| \|\obj\|}{\lambda} \right) \frac{\|\obj\|}{\|\vec{r}\|} r_i~.
\end{equation}

To simplify the analysis, let $\Delta_i = \|\obj\| - \obj_i \geq 0$ be the objective deficit. Substituting $\obj_i = \|\obj\| - \Delta_i$ yields:
\begin{equation}
    \|\obj\| - \Delta_i \geq \left( 1 - \frac{\|\vec{r}\| \|\obj\|}{\lambda} \right) \frac{\|\obj\|}{\|\vec{r}\|} r_i~.
\end{equation}

Expanding the right side and moving all terms to the left forms a quadratic inequality with respect to $\|\obj\|$:
\begin{equation}
    \frac{r_i}{\lambda} \|\obj\|^2 + \left( 1 - \frac{r_i}{\|\vec{r}\|} \right) \|\obj\| - \Delta_i \geq 0~.
\end{equation}

Solving this quadratic inequality for its non-negative root gives the exact necessary lower bound for the total loss norm $\|\obj\|$:
\begin{equation}
    \|\obj\| \geq \frac{\lambda}{2 r_i} \left[ -\left( 1 - \frac{r_i}{\|\vec{r}\|} \right) + \sqrt{ \left( 1 - \frac{r_i}{\|\vec{r}\|} \right)^2 + \frac{4 r_i \Delta_i}{\lambda} } \right]~.
\end{equation}

We now analyze the behavior of this necessary condition as the preference vector concentrates on objective $i$, meaning $r_i \to 1$. Because $\vec{r}$ is on the probability simplex ($\sum r_j = 1$ and $r_j \geq 0$), $r_i \to 1$ implies that all other components $r_j \to 0$, which consequently means $\|\vec{r}\| \to 1$. Applying this limit to the exact bound:
\begin{align}
    &\lim_{r_i \to 1} \frac{\lambda}{2 r_i} \left[ -\left( 1 - \frac{r_i}{\|\vec{r}\|} \right) + \sqrt{ \left( 1 - \frac{r_i}{\|\vec{r}\|} \right)^2 + \frac{4 r_i \Delta_i}{\lambda} } \right] \notag \\
    &\quad= \frac{\lambda}{2} \sqrt{\frac{4 \Delta_i}{\lambda}} \notag = \frac{\lambda}{2} \cdot \frac{2\sqrt{\Delta_i}}{\sqrt{\lambda}} \notag \\
    &\quad= \sqrt{\lambda \Delta_i}~.
\end{align}

Therefore, as $r_i \to 1$, the necessary condition for Pareto stationarity simplifies to:
\begin{equation}
    \|\obj\| \geq \sqrt{\lambda \Delta_i}~.
\end{equation}

Since both $\|\obj\|$ and $\sqrt{\lambda \Delta_i}$ are non-negative, we can square both sides and isolate the hyperparameter $\lambda$:
\begin{equation}
    \|\obj\|^2 \geq \lambda \Delta_i \implies \lambda \leq \frac{\|\obj\|^2}{\Delta_i}~.
\end{equation}

Substituting the objective deficit $\Delta_i = \|\obj\| - \obj_i$ back into the inequality:
\begin{equation}
    \lambda \leq \frac{\|\obj\|^2}{\|\obj\| - \obj_i}~.
\end{equation}

If this bound is violated, i.e., $\lambda > \frac{\|\obj\|^2}{\|\obj\| - \obj_i}$, no value of $\|\obj\|$ can maintain $\vec{r}' \geq 0$, forcing the effective weighting vector $\vec{r}'$ to contain negative components. Since our assumption guarantees that $\vec{r}'$ is the unique vector satisfying $J_\obj^T \vec{r}' = 0$, the presence of negative components implies that no non-negative preference vector $\vec{r} \geq 0$ can exist. Thus, when $r_i \to 1$ and $\lambda > \frac{\|\obj\|^2}{\|\obj\| - \obj_i}$, the stationary point is strictly not Pareto stationary. 
\end{proof}

\subsection{Stochastic weighting as implicit regularization}
\label{app:regulariaztion}
\citet{zhang2026neural} show that, under standard assumptions on the data and architecture, a two-layer neural network trained via layer-wise gradient descent on a Gaussian multi-index model $f(x) = g(Ux)$ learns the hidden subspace $U$ through a power method dynamic. Specifically, they consider
\begin{itemize}
    \item inputs $x \sim \mathcal{N}(0, I_d)$
    \item a two-layer network $h(x) = \sum_j a_j \sigma(w_j^\top x)$ with %
    a locally quadratic activation ($\sigma'(0) = 0$, $\sigma''(0) = 1$, bounded third derivative) %
    \item small isotropic initialization $w_j^{(0)} \sim \mathcal{N}(0, \sigma_w^2 I_d)$ with symmetric pairing of neurons such that the network output is zero at initialization and 
    \item fixed outer weights $a_j\in \{\pm 1\}$ during the first training stage.
\end{itemize}  
Near initialization the network output stays close to zero and $\sigma'(w_j^\top x)\approx\sigma''(0)w_j^\top x$, so the gradient w.r.t. $w_j$ is $a_j\sigma''(0)\hat{\Sigma_l}w_j$ where 
\begin{equation}
    \hat{\Sigma}_l = \frac{1}{n}\sum_{i=1}^n l'(0, y_i) x_i x_i^\top
\end{equation}
is the \emph{empirical discovery matrix}.  
Under these conditions, gradient descent with learning rate $\eta$ amd weight decay $1/\eta$ cancels the identity part of the update and yields a power iteration on $\hat{\Sigma}_l$ (using $\sigma''(0) = 1$):
\begin{equation}
    \label{eq:power-iteration}
    w_j^{(t+1)} \approx -\eta a_j\hat{\Sigma}_lw_j^{(t)} \quad  \text{so that} \quad w_j^{(T)} \approx (-\eta a_j)^T\hat{\Sigma}_l^T w_j^{(0)}.
\end{equation}
This implies that after $T$ steps, the expected fraction of a neuron's energy aligned with feature direction $u_k$ scales as $\lambda_k^{2T} / (\sum_i \lambda_i^{2T})$, where $\lambda_k$ is the corresponding eigenvalue of $\hat{\Sigma}_l$.
Features with larger eigenvalues exponentially dominate, and finite-sample noise is filtered out through the spectral gap.
We measure how much neuron $j$ has aligned with $u_k$ after $T$ steps by its \emph{feature energy} $\mathcal{E}_k (T) := \langle w_j^{(T)}, u_k \rangle^2$. Since $\mathcal{E}_k (T)$ is a product of $T$ random factors, we summarize it by its geometric mean $Signal_k(T) = \exp \big(\mathbb{E}[\log \mathcal{E}_k(T)]\big)$.

Leveraging the results from \citet{zhang2026neural}, we can show that in MTL stochastic task weights regularize against domination of one task, where the regularization strength is controlled by the variance of the Dirichlet distribution we sample from. \\

\begin{theorem}
\label{thm:variance_penalty}
Consider two tasks $A$ and $B$ together with the following assumptions: 
\begin{itemize}
    \item The discovery matrices $\Sigma_A, \Sigma_B$ share the same orthonormal eigenbasis $U$ with task-specific eigenvalues $\lambda_{k,A}$ and $\lambda_{k,B}$ for each feature direction $u_k$.
    \item At each gradient descent step $t$, the tasks are combined with stochastic weights $r_A^{(t)}, r_B^{(t)}$ satisfying $r_A^{(t)} + r_B^{(t)} = 1$, drawn i.i.d.\ from a distribution with expectations $\bar{r}_A, \bar{r}_B$ and shared variance $V_\alpha(r) := \text{Var}(r_A) = \text{Var}(r_B)$.
    \item Assuming $\lambda_{k,A}, \lambda_{k,B} \geq 0$ with $\bar{\lambda}_k>0$; let $\bar{\lambda}_k := \bar{r}_A \lambda_{k,A} + \bar{r}_B \lambda_{k,B}$ denote the expected eigenvalue and $\Delta_k := (\lambda_{k,A} - \lambda_{k,B})/\bar{\lambda}_k$ denote the relative task disagreement.
    \item Initialization $w_j^{(0)}$ is isotropic and independent of the task weights.
    \item The outer weights $a_j \in \{\pm1\}$ are shared by both tasks and fixed.
\end{itemize}
Then, under the power iteration regime of \citet{zhang2026neural}, the relative learning ratio between features $k$ and $m$ after $T$ steps satisfies 
\begin{equation}\label{eq:ratio-exact}
    \rho_{k,m}(T) = \left(\frac{\bar{\lambda}_k}{\bar{\lambda}_m}\right)^{2T}\exp\bigl(-2T\,[\phi_k - \phi_m]\bigr),
\end{equation}
with $\phi_k$ as defined in~\eqref{eq:stochastic-penalty}. If, in addition, $V_\alpha(r)\Delta_k^2 \ll 1$ (small fluctuations), then
\begin{align}\label{eq:rho}
    \begin{split}
         \rho_{k,m}(T)
        :=  \frac{Signal_k(T)}{Signal_m(T)} = \frac{\exp\big(\mathbb{E}\bigl[\log|\langle w_j^{(T)}, u_k\rangle|^2\bigr]\big)}{\exp\big(\mathbb{E}\bigl[\log|\langle w_j^{(T)}, u_m\rangle|^2\bigr]\big)} \\
        \approx \left(\frac{\bar{\lambda}_k}{\bar{\lambda}_m}\right)^{2T}
        \exp\Big({-T\cdot V_\alpha(r)\cdot \big[\Delta_k^2 - \Delta_m^2\big]}\Big).
    \end{split}
\end{align}
\end{theorem}

\begin{proof}
Under stochastic task weighting, the multi-task discovery matrix at step $t$ is 
\begin{equation}
    \Sigma_{\text{MT}}^{(t)} = r_A^{(t)}\Sigma_A + r_B^{(t)}\Sigma_B
\end{equation}
where $\Sigma_A, \Sigma_B$ are the empirical discovery matrices for task $A$ and $B$ respectively. 
Since both tasks share the eigenbasis $U$, the operator acts as a scalar on each eigendirection
\begin{equation}
    \lambda_k^{(t)} = r_A^{(t)}\lambda_{k,A} + r_B^{(t)}\lambda_{k,B}~.
\end{equation}
Writing the weight fluctuation as $\delta_t := r_A^{(t)} - \bar{r}_A$ (so that $r_B^{(t)} - \bar{r}_B = -\delta_t$), we obtain
\begin{equation}
    \lambda_k^{(t)} = \bar{\lambda}_k + \delta_t(\lambda_{k,A} - \lambda_{k,B})~.
\end{equation}
Note that $\lambda_k^{(t)}$ is a convex combination of $\lambda_{k,A},\lambda_{k,B} \geq 0$, so $\lambda_k^{(t)} > 0$ almost surely and its logarithm below is well defined.

Since $u_k$ is an eigenvector of both $\Sigma_A$ and $\Sigma_B$, it is an eigenvector of $\Sigma_{\text{MT}}^{(t)}$ for every realization of the task weights. Projecting the power iteration from Eq.~\eqref{eq:power-iteration}, using $\Sigma_{\text{MT}}^{(t)}$ in place of $\hat{\Sigma}_l$, onto $u_k$ gives the following recursion 
\begin{equation}
    \langle w_j^{(t)}, u_k \rangle = -\eta a_j \lambda_k^{(t)} \langle w_j^{(t-1)}, u_k \rangle
\end{equation}
for the $t$-th update.
Iterating from $t=1,\ldots,T$, the projection of the weight vector onto $u_k$ after $T$ steps factors as 
\begin{equation}
    \langle w_j^{(T)}, u_k \rangle = (-\eta a_j)^T \gamma_k \prod_{t=1}^T \lambda_k^{(t)},
\end{equation}
where $\gamma_k := \langle w_j^{(0)}, u_k \rangle$ is the initial projection. The feature energy aligned with feature $k$ is the square of this projection and since $|a_j|=1$:
\begin{equation}
    \mathcal{E}_k (T) = |\langle w_j^{(T)}, u_k \rangle|^2 = \eta^{2T} \gamma_k^2 \prod_{t=1}^T \bigl(\lambda_k^{(t)}\bigr)^2.
\end{equation}
Moving to log-space, we obtain:
\begin{equation}
    \label{eq:log-enegery}
    \log\mathcal{E}_k (T)  = 2T\log\eta + \log \gamma_k^2 + 2\sum_{t=1}^T \log \lambda_k^{(t)}.
\end{equation}
Taking expectations, $\mathbb{E}[\log \lambda_k^{(t)}]$ is the same for every $t$ since the task weights are identically distributed. We define the \emph{stochastic penalty} of feature $k$ as 
\begin{equation}
    \label{eq:stochastic-penalty}
    \phi_k := \log \bar{\lambda}_k - \mathbb{E}[\log \lambda_k^{(t)}] \geq 0
\end{equation}
which is non-negative by Jensen's inequality, with equality if and only if $\lambda_k^{(t)}$ is almost surely constant, e.g. for a shared feature with $\lambda_{k,A} = \lambda_{k,B}$.
Substituting $\mathbb{E}[\log \lambda_k^{(t)}]$ according to Eq.~\eqref{eq:stochastic-penalty} and exponentiating both sides gives the geometric mean signal for feature $k$:
\begin{equation}
\label{eq:signal}
    Signal_k(T) = \exp\big(\mathbb{E}[\log \mathcal{E}_k(T)]\big) = \exp\big(\mathbb{E}[\log\gamma_k^2]\big)\cdot (\eta\bar{\lambda}_k)^{2T} \cdot \exp\left(-2T\phi_k\right).
\end{equation}
Forming the ratio $\rho_{k,m}(T) := Signal_k(T)/Signal_m(T)$, the learning rate $\eta^{2T}$ cancels identically, and the initialization terms cancel since $\mathbb{E}[\log \gamma_k^2] = \mathbb{E}[\log\gamma_m^2]$. Combining the exponential terms via $e^A/e^B = e^{A-B}$ yields
\begin{equation}
    \label{eq:ratio-exact}
    \rho_{k,m}(T)= \left(\frac{\bar{\lambda}_k}{\bar{\lambda}_m}\right)^{2T}
        \exp\!\Big({-2T \big[\phi_k - \phi_m\big]}\Big).    
\end{equation}
Thus it remains to evaluate $\phi_k$ and $\phi_m$.
Substituting $\lambda_k^{(t)} = \bar{\lambda}_k + \delta_t(\lambda_{k,A} - \lambda_{k,B})$ and applying the second-order Taylor expansion 
\begin{equation}
    \log(\bar{\lambda}_k + x) \approx \log\bar{\lambda}_k + \frac{x}{\bar{\lambda}_k} - \frac{x^2}{2\bar{\lambda}_k^2}
\end{equation}
with $x = \delta_t(\lambda_{k,A} - \lambda_{k,B})$, each summand becomes:
\begin{equation}
    \log\lambda_k^{(t)} \approx \log\bar{\lambda}_k + \delta_t\,\frac{\lambda_{k,A} - \lambda_{k,B}}{\bar{\lambda}_k} - \frac{\delta_t^2}{2}\left(\frac{\lambda_{k,A} - \lambda_{k,B}}{\bar{\lambda}_k}\right)^2.
\end{equation}
Taking expectations term by term, the zero-order term contributes $\log\bar{\lambda}_k$; the first-order term vanishes since $\mathbb{E}[\delta_t] = 0$ ; the second-order term contributes $-\tfrac{1}{2}V_\alpha(r)\Delta_k^2$, using $\mathbb{E}[\delta_t^2] = V_\alpha(r)$ and the definition $\Delta_k := (\lambda_{k,A} - \lambda_{k,B})/\bar{\lambda}_k$. Thus
\begin{equation}
    \label{eq:phi-approx}
    \phi_k \approx \frac{1}{2} V_\alpha(r) \Delta_k^2 \quad \text{and} \quad \phi_m \approx \frac{1}{2} V_\alpha(r) \Delta_m^2
\end{equation}
with $V_\alpha(r) \Delta_k^2, V_\alpha(r) \Delta_m^2 \ll 1$ under the small fluctuation assumption.

Subsituting Eq.~\eqref{eq:phi-approx} in Eq.~\eqref{eq:ratio-exact} yields Eq.~\eqref{eq:rho}.
\end{proof}

\begin{remark}
Equation~\eqref{eq:rho} decomposes the competition between any two features into two multiplicative forces.
The first factor, $(\bar{\lambda}_k/\bar{\lambda}_m)^{2T}$, is a \emph{deterministic race}: it exponentially favors whichever feature has the larger average eigenvalue across tasks, and it is exactly the dynamics that would arise under fixed (non-stochastic) task weighting.
The second factor, $\exp\big(-T\cdot V_\alpha(r)[\Delta_k^2 - \Delta_m^2]\big)$, is a \emph{stochastic penalty}: it exponentially suppresses whichever feature has higher relative task disagreement $\Delta_k$.
This penalty arises because multiplicative noise always reduces geometric growth rates, the same mechanism by which SGD's minibatch noise implicitly regularizes in the single-task setting~\citep{smith2018sgd}. Crucially, shared features ($\Delta_k \approx 0$) are unaffected by the penalty, while task-specific features ($|\Delta_k| \gg 0$) are actively suppressed, preventing the dominant task from hijacking the network's representational capacity.
The sampling variance $V_\alpha(r)$ controls the strength of this regularization, playing a role analogous to the noise scale $\eta/B$ in SGD implicit regularization~\citep{smith2018batchsize}. \\
\end{remark}

\begin{remark}[Regularization overcoming task domination.]
\label{rem:domination}
Suppose feature $k$ has a larger average eigenvalue than feature $m$ ($\bar{\lambda}_k > \bar{\lambda}_m$), but feature $k$ is task-specific while feature $m$ is shared ($|\Delta_k| > |\Delta_m| \approx 0$).
The two factors in~\eqref{eq:rho} pull in opposite directions.
Feature $m$ wins the competition (i.e., $\rho_{k,m}(T) < 1$) when the stochastic penalty dominates.
Since both factors scale linearly in $T$ in the exponent, the step count cancels and the critical variance at which the transition occurs is:
\begin{equation}\label{eq:critical_variance}
    V_\alpha^\star(r) = \frac{2\log(\bar{\lambda}_k / \bar{\lambda}_m)}{\Delta_k^2 - \Delta_m^2}.
\end{equation}
For $V_\alpha(r) > V_\alpha^\star(r)$, the network aligns more strongly with the weaker but shared feature than with the dominant but task-specific one.
For $V_\alpha(r) < V_\alpha^\star(r)$, the deterministic advantage prevails and the task-specific feature dominates, leading to negative transfer.
When sampling task weights from a symmetric Dirichlet distribution $\mathrm{Dir}(\alpha,\alpha)$, $V_\alpha(r) = 1/(8\alpha + 4)$ for the two-task case, so lower concentration $\alpha$ increases the regularization strength.

The threshold~\eqref{eq:critical_variance} relies on the second-order approximation~\eqref{eq:rho} and therefore inherits the small-fluctuation assumption $V_\alpha(r)\Delta_k^2 \ll 1$ of Theorem~\ref{thm:variance_penalty}. Within this regime, $V_\alpha(r) > V_\alpha^\star(r)$ requires
\begin{equation}
    2\log(\bar{\lambda}_k/\bar{\lambda}_m) < V_\alpha(r)\,(\Delta_k^2 - \Delta_m^2) \ll 1,
\end{equation}
so~\eqref{eq:critical_variance} only describes reversals between features whose average eigenvalues are close, $\bar{\lambda}_k/\bar{\lambda}_m \approx 1$.

To overturn a large eigenvalue gap, the weights must fluctuate strongly, i.e.\ the variance must be large, which lies outside the validity of the Taylor expansion. In that case we use the exact identity~\eqref{eq:ratio-exact}, which requires no small-fluctuation assumption and according to which $\rho_{k,m}(T) < 1$ if and only if
\begin{equation}
    \log(\bar{\lambda}_k/\bar{\lambda}_m) < \phi_k - \phi_m,
\end{equation}
again independently of $T$. For example, let feature $k$ be specific to task $A$ ($\lambda_{k,B} = 0$), let feature $m$ be shared ($\lambda_{m,A} = \lambda_{m,B}$) and let $(r_A^{(t)}, r_B^{(t)}) \sim \mathrm{Dir}(\alpha_A, \alpha_B)$ with $\alpha_0 := \alpha_A + \alpha_B$. Then $\lambda_k^{(t)} = r_A^{(t)}\lambda_{k,A}$ and $\lambda_m^{(t)} = \bar{\lambda}_m$, so that $\phi_m = 0$ and
\begin{equation}
    \phi_k = \log\bar{r}_A - \mathbb{E}\bigl[\log r_A^{(t)}\bigr].
\end{equation}
Here $\bar{r}_A = \alpha_A/\alpha_0$ and $\mathbb{E}[\log r_A^{(t)}] = \psi(\alpha_A) - \psi(\alpha_0)$, where $\psi$ is the digamma function, so the penalty has the closed form
\begin{equation}
    \phi_k = \log\frac{\alpha_A}{\alpha_0} - \psi(\alpha_A) + \psi(\alpha_0).
\end{equation}
The shared feature therefore wins if and only if $\bar{\lambda}_k/\bar{\lambda}_m < \exp(\phi_k)$. In the symmetric case $\mathrm{Dir}(\alpha,\alpha)$, the penalty reduces to $\phi_k = \psi(2\alpha) - \psi(\alpha) - \log 2$. It decreases monotonically in $\alpha$, diverges as $\alpha \to 0$, and approaches $1$ as $\alpha \to \infty$, where it agrees with~\eqref{eq:critical_variance}. If feature $k$ is only approximately task-specific or feature $m$ only approximately shared, these expressions hold approximately. Consequently, for any eigenvalue gap $\bar{\lambda}_k/\bar{\lambda}_m$, a sufficiently small concentration lets the shared feature win. However, strong regularization comes at a cost: as $\alpha \to 0$ the weights concentrate on the vertices of the simplex, training degenerates into alternating single-task updates, and every feature with $\phi_k > 0$ is slowed relative to fixed weighting. \\
\end{remark}

\newpage
\begin{remark}[No regularization effect.]
If both features exhibit similar levels of task disagreement ($\Delta_k \approx \Delta_m$), then $\Delta_k^2 - \Delta_m^2 \approx 0$ and the stochastic penalty cancels out of the ratio regardless of $V_\alpha(r)$.
The competition reduces to the pure deterministic race $\rho_{k,m}(T) \approx (\bar{\lambda}_k/\bar{\lambda}_m)^{2T}$, and the stronger feature dominates exponentially.
This is expected: the variance penalty targets \emph{differential} disagreement between features, not disagreement per se.
When both features are equally task-specific (or equally shared), the stochastic weighting provides no selective pressure. \\
\end{remark}

\begin{proposition}[Worst-case generalization to $J$ tasks]
Consider $J$ tasks that share the same orthonormal eigenbasis $U$ for their discovery matrices, with task-specific eigenvalues $\lambda_{k,1},\ldots,\lambda_{k,J} \geq 0$ for each feature direction $u_k$. At each gradient descent step $t$, the tasks are combined with stochastic weights $r_1^{(t)},\ldots,r_J^{(t)}$ satisfying $\sum_{i=1}^J r_i^{(t)} = 1$, drawn i.i.d.\ from a distribution with expectations $\bar{r}_i$ and shared variance $V_\alpha(r) := \text{Var}(r_1)$. Let $\bar{\lambda}_k := \sum_{i=1}^J r_i^{(t)}\lambda_{k,i}$ denote the expected eigenvalue and let 
\begin{equation*}
    \lambda^{max}_k := \max_i \lambda_{k,i} \quad\quad    \lambda^{min}_k := \min_i \lambda_{k,i} \quad\quad \Delta_k := \frac{\lambda^{max}_k -\lambda^{min}_k }{\bar{\lambda}_k}
\end{equation*}
denote the eigenvalue of most and least favored  task for feature $k$ and the worst-case relative disagreement between them. Assuming $\Delta_k,\Delta_m\ll1$, the relative learning ratio is bounded by 
\begin{equation}
    \label{eq:worst_case_bound}
    \left(\frac{\bar{\lambda}_k}{\bar{\lambda}_m}\right)^{2T} \exp\!\left({-\frac{T}{4} \Delta_k^2}\right) \lesssim \rho_{k,m}(T) \lesssim \left(\frac{\bar{\lambda}_k}{\bar{\lambda}_m}\right)^{2T} \exp\!\left({\frac{T}{4} \Delta_m^2}\right)
\end{equation}
where $\lesssim$ denotes inequality up to the second-order approximation.
\end{proposition}
\begin{proof}
    Since $\lambda_k^{(t)}$ is a convex combination of $\{\lambda_{k,i}\}_i$ it takes values in $[\lambda^{min}_k, \lambda^{max}_k]$ almost surely and $\mathbb{E}[\lambda_k^{(t)}] = \bar{\lambda}_k$.  By Popoviciu's inequality,
    \begin{equation}
        \label{eq:var-upper-bound}
        \text{Var}\left(\lambda_k^{(t)}\right) \leq \frac{(\lambda^{max}_k -\lambda^{min}_k)^2}{4} = \frac{\Delta_k^2\bar{\lambda}_k^2}{4}~.
    \end{equation}
    The proof of Theorem~\ref{thm:variance_penalty} up to Eq.~\eqref{eq:ratio-exact} uses only the shared eigenbasis, positivity of $\lambda_k^{(t)}$ and identically distributed weights and thus applies to $J$ tasks in the same manner. Hence we can use Eq.~\eqref{eq:ratio-exact} and bound $\phi_k$ from both sides. 
    From below, $\phi_k \geq 0$ by Jensen's inequality. From above, we can repeat the Taylor expansion of Theorem~\ref{thm:variance_penalty} with $\lambda_k^{(t)}-\bar{\lambda}_k$ in place of $\delta_t(\lambda_{k,A}-\lambda_{k,B})$ gives 
    \begin{equation}
        \phi_k \approx \frac{1}{2\bar{\lambda}_k^2} \text{Var}\left(\lambda_k^{(t)}\right) \leq \frac{\Delta_k^2}{8}
    \end{equation}
    using Popoviciu's bound in~\eqref{eq:var-upper-bound}. Thus $0\leq\phi_k\lesssim \Delta_k^2/8$, plugging this to Eq.~\eqref{eq:signal} yields
    \begin{equation}
        \exp\big(\mathbb{E}[\log\gamma_k^2]\big)\cdot (\eta\bar{\lambda}_k)^{2T} \cdot \exp\left(-\frac{T}{4}\Delta_k^2\right) \lesssim Signal_k(T) \leq
        \exp\big(\mathbb{E}[\log\gamma_k^2]\big)\cdot (\eta\bar{\lambda}_k)^{2T} 
    \end{equation}
    The same bounds hold for feature $m$. To bound the ratio $\rho_{k,m}(T) = Signal_k(T)/Signal_m(T)$, we combine the upper bound on one signal with the lower bound on the other. The prefactors $\exp\big(\mathbb{E}[\log\gamma_k^2]\big)$ and $\exp\big(\mathbb{E}[\log\gamma_m^2]\big)$ cancel because $\mathbb{E}[\log\gamma_k^2] = \mathbb{E}[\log\gamma_m^2]$ by the isotropic initialization, and $\eta^{2T}$ cancels identically. For the upper bound, dividing the upper bound on $Signal_k(T)$ by the lower bound on $Signal_m(T)$ gives
    \begin{equation}
        \rho_{k,m}(T) \lesssim \frac{(\eta\bar{\lambda}_k)^{2T}}{(\eta\bar{\lambda}_m)^{2T}\exp\left(-\frac{T}{4}\Delta_m^2\right)}
        = \left(\frac{\bar{\lambda}_k}{\bar{\lambda}_m}\right)^{2T}\exp\left(\frac{T}{4}\Delta_m^2\right).
    \end{equation}
    For the lower bound, dividing the lower bound on $Signal_k(T)$ by the upper bound on $Signal_m(T)$ gives
    \begin{equation}
        \rho_{k,m}(T) \gtrsim \frac{(\eta\bar{\lambda}_k)^{2T}\exp\left(-\frac{T}{4}\Delta_k^2\right)}{(\eta\bar{\lambda}_m)^{2T}}
        = \left(\frac{\bar{\lambda}_k}{\bar{\lambda}_m}\right)^{2T}\exp\left(-\frac{T}{4}\Delta_k^2\right).
    \end{equation}
    Together, these two inequalities yield~\eqref{eq:worst_case_bound}. 
\end{proof}

\begin{remark}[Extension and Scope]
The theory presented in Theorem~\ref{thm:variance_penalty} extends to broader architectures and settings through a single key structural property: that early-stage gradient descent drives each neuron's parameters via a \emph{power iteration} on a label-weighted second-order operator, such that the projection onto any joint feature direction $u_k$ factors multiplicatively as $$\langle w_j^{(T)}, u_k\rangle = \eta^T \gamma_k \prod_{t=1}^T \lambda_k^{(t)}.$$ We discuss the conditions required by Theorem~\ref{thm:variance_penalty} and how they generalize below. 

\textbf{Fixed outer weights.} MOSEL trains the representation network $f_\theta$ and the preference alignment network $g_\phi$ with learning rates satisfying $\eta_\phi \ll \eta_\theta$, following the Stackelberg timescale separation of \citet{pmlr-v119-fiez20a}. Thus $g_\phi$ is approximately stationary over the timescale on which $\theta$ updates, corresponding to fixed outer weights during the feature learning phase while $f_\theta$ executes the power iteration. 

\textbf{Power iteration in representation space.} Under the layer-wise decoupling of \citet{dandi2026deep}, which holds for smooth activations with $\sigma''(0) \neq 0$ and vanishing initialization scaling, each layer $\ell$ independently performs a power iteration on a label-weighted moment operator $$\hat{C}^{(\ell)} = \tfrac{1}{n}\sum_{\mu=1}^n y_\mu z_{\ell-1}(x_\mu) z_{\ell-1}(x_\mu)^\top.$$
Linear scalarization under stochastic preference sampling $r^{(t)} \sim \text{Dir}(\alpha)$ produces at each layer and step an effective operator $\hat{C}^{(\ell,t)} = \sum_i r_i^{(t)} \hat{C}^{(\ell)}_i$, a stochastic convex combination of the per-objective operators. This layer-wise analysis extends naturally to convolutional and other structured architectures, where the spectral operators can be equivalently defined over the channel space. 

\textbf{Shared eigenbasis.} Requiring a common eigenbasis is equivalent to requiring that the per-task discovery matrices $\Sigma_A$ and $\Sigma_B$ commute. Shared inputs alone do not guarantee this. It holds when the tasks depend on a common set of orthonormal feature directions and each task's discovery matrix is diagonal in that set. This includes tasks whose relevant subspaces coincide, tasks whose subspaces are mutually orthogonal, and tasks that draw on different subsets of a shared feature dictionary. It generally fails when the task subspaces overlap at oblique angles. The assumption is therefore most natural in multi-task learning, where all tasks are defined on the same inputs and differ only in their targets. In fairness settings it is more restrictive and does not hold in general. We further require the tasks to agree in sign on each feature's eigenvalue. If they do not, the combined eigenvalue $\lambda_k^{(t)}$ can reach zero, and the corresponding direction is erased rather than merely slowed down.

\textbf{Loss function.} The power-iteration mechanism requires only that early training iterates some label-dependent operator; the loss function affects the dynamics solely through the form of this operator. Under general losses, the relevant spectral object generalizes to a transformed correlation operator $\hat{\Sigma}_g = \tfrac{1}{n}\sum_{i=1}^n g(y_i) x_i x_i^\top$, where $g(y)$ is determined by the loss gradient~\citep{dandi2026deep}. The core stochastic mitigation of task domination is therefore decoupled from the squared loss. While an explicit multiplicative formula for classification losses remains an open challenge, the high-dimensional analysis of \citet{ben2024high} provides rigorous support: SGD trajectories concentrate layer by layer on the low-rank outlier eigenspaces of the empirical Hessian and gradient matrices, confirming that an eigenvalue-ordered selection process governs representations under classification losses and suggesting our analysis extends to classification regimes in future work.
\end{remark}

\section{Experimental details}
\label{app:exp}
\subsection{Experimental setup}
\paragraph{General settings.}
All datasets use a 70/10/20 train/validation/test split. We use early-stopping based on validation hypervolume across all methods. For a fair comparison, all methods share the same architecture. \ourmethod~additionally requires splitting the network into a representation learning and a preference alignment component, as described in the respective subsections below.
For COSMOS, we use all hyperparameters as reported in the original paper. 

Unless stated otherwise, all methods share the same training hyperparameters, including the Dirichlet sampling parameter $\vec{\alpha}$, number of epochs, and batch size, which mostly follow the setup in~\cite{ruchte2021scalable}.  All methods use Adam~\cite{kingma2014adam} as the optimizer. For \ourmethod, we use separate Adam optimizers for the leader and follower. To ensure convergence to an equilibrium point, the leader learning rate must be smaller than the follower's, we found that one order of magnitude difference works best. Based on that, we select the learning rate combination yielding the best validation hypervolume. Learning rates are decayed by a factor of 0.1 at dataset-specific milestones. For all other methods, hyperparameters are selected via grid search over validation hypervolume. 
The smooth Chebychev scalarization \cite{lin2024smooth} additionally requires a smoothing parameter $\tau$, set to 0.01 across all datasets, and the utopia point $\vec{u}^\star$. Since $\vec{u}^\star$ is generally unknown, we approximate it with a running minimum over objectives observed during training.

All experiments were conducted on an internal computing cluster with an automatic job assignment system, thus  we cannot ensure the specific CPU and GPU used for each particular experiment. However, all jobs were submitted with defined resource requirements: multi-dataset experiments requested 1 GPU and 1 CPU, FairFace and CelebA experiments requested 1 GPU and 4 CPUs, and tabular experiments requested no GPU and 2 CPUs. The required memory for each experiment is reported in the respective tables.

Our implementation can be found at~\url{https://github.com/ElinaCelik/mosel}.

\subsubsection{Tabular datasets}
\paragraph{Model.}  For the fairness use case, we consider three tabular datasets: Adult, COMPAS and Default. Following prior work~\cite{ruchte2021scalable}, we use an MLP with two hidden layers and hidden dimensions 60 and 25:
\texttt{Linear($d$, 60) $\rightarrow$ ReLU $\rightarrow$ Linear(60, 25) $\rightarrow$ ReLU $\rightarrow$ Linear(25, 1)}. 
\vspace{-5pt}
\paragraph{Hyperparameters.} The model is trained for 50 epochs with a batch size of 256, using $\vec{\alpha} = [0.5, 0.5]$ for Adult and COMPAS and  $\vec{\alpha} =[0.1,0.5]$ for Default, following the values of~\cite{ruchte2021scalable}.
The COSMOS learning rate is $10^{-3}$ for all tabular datasets without a learning rate scheduler. For \ourmethod, we use leader and follower learning rates of $[10^{-3},10^{-2}]$ for Adult and Default, and $[10^{-2},10^{-1}]$ for COMPAS. The learning rate is decayed by a factor of 0.1 at epochs 15 and 40 for Adult, epochs 30 and 40 for COMPAS, and epochs 30 and 45 for Default.
\vspace{-5pt}
\paragraph{Splitting the architecture.} This architecture provides two natural split points, namely before the second or third linear layer. Empirically, splitting before the last linear layer yields insufficient capacity in the preference alignment network to meaningfully incorporate the preference vector. Thus we split before the second linear layer across all datasets.

\subsubsection{Multi-Image classification}
\paragraph{Model.} For the Multi-MNIST, Multi-Fashion and Multi-Fashion+MNIST datasets we adopt the standard LeNet architecture~\cite{lecun1999object}. 
\vspace{-5pt}
\paragraph{Hyperparameters.} Following the values of~\cite{ruchte2021scalable}, the model is trained for 100 epochs with a batch size of 256 and $\vec{\alpha} = [1.2,1.2]$. The COSMOS learning rate is $10^{-3}$, decayed by a factor of 0.1 at epochs 20, 40, 80, and 90. For \ourmethod, we use leader and follower learning rates of $[10^{-4},10^{-3}]$ with the same scheduler milestones. A weight decay of $10^{-3}$ is applied to the leader to keep the impact of leader and follower balanced and improve training stability.
\vspace{-5pt}
\paragraph{Splitting the architecture.} We split the network immediately after flattening, using the convolutional trunk for representation learning which produces a 720-dimensional embedding $z$, that is then passed to the subsequent shared linear layer and two task-specific heads that form the preference alignment network.

\subsubsection{CelebA}
\paragraph{Model.} For CelebA, we use the pretrained EfficientNet-B4~\cite{pmlr-v97-tan19a}.
\paragraph{Hyperparameters.} Again, we follow the values of~\cite{ruchte2021scalable}, thus the model is finetuned for 30 epochs with a batch size of 32 and $\vec{\alpha} = [1,1]$. The COSMOS learning rate is $5 \times10^{-4}$, for \ourmethod, we use leader and follower learning rates of $[5 \times10^{-5},5 \times10^{-4}]$.
\paragraph{Splitting the architecture.}The modular block structure allows for many candidate split points. We select the split parameter by cross-validation on validation hypervolume. The optimal split was found after the 16th MBConv block, such that the convolutional stem and first 16 MBConv blocks form the representation learning network. The remaining MBConv blocks together with the convolutional head, pooling, and two-layer MLP task heads form the preference alignment network. \vspace{-3pt}

\subsubsection{FairFace}
\paragraph{Model.} For FairFace we use a pretrained ResNet-34~\cite{he2015deep} following~\cite{karkkainenfairface}. \vspace{-3pt}
\paragraph{Hyperparameters.} Since FairFace was not used in~\cite{ruchte2021scalable}, all hyperparameters are selected via cross-validation on validation hypervolume. The model is finetuned for 30 epochs with a batch size of 128 and $\vec{\alpha} = [0.5, 0.5]$. The COSMOS hyperparameters are $\lambda=0.1$ and learning rate of $10^{-4}$, decayed by a factor of 0.1 at epochs 20. For \ourmethod we use leader and follower learning rates of $[10^{-5},10^{-4}]$ with the same scheduler milestones.\vspace{-3pt}
\paragraph{Splitting the architecture.} %
Possible split points are restricted to boundaries between residual stages (layer1 - layer4), which gives a total of four possible split poitns. Again, we cross-validate this parameter on validation hypervolume. The optimal split was found after layer3, thus the convolutional stem together with the first three residual stages form the representation learning network, and layer4 together with the two-layer MLP task heads form the preference alignment network. \vspace{-3pt}

\subsection{Hypervolume computation}
\label{app:hv}
\vspace{-3pt}
\begin{wrapfigure}{r}{0.5\textwidth}
    \centering
    \includegraphics[width=\linewidth]{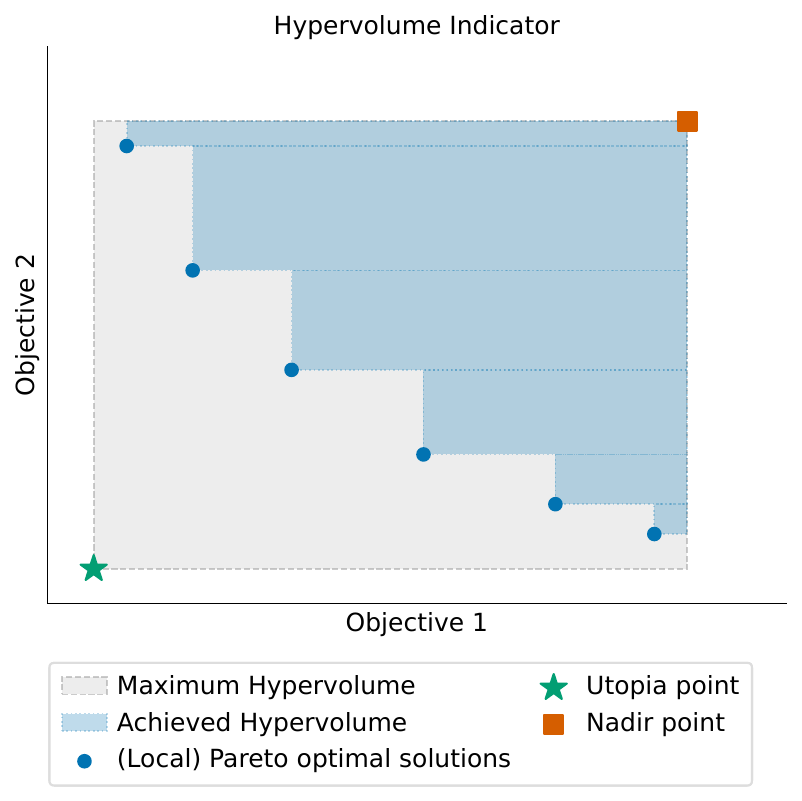}
    \caption{Hypervolume computation}
    \label{fig:hv}
\end{wrapfigure}
The hypervolume indicator measures the volume of the objective space dominated by a set of solutions and bounded by a reference point~\cite{zitzler2007hypervolume}. As illustrated in Fig.~\ref{fig:hv}, we compute it stepwise as a sum of rectangles, where each solution on the approximated Pareto front contributes a rectangular region bounded by its objective values and the nadir point. To enable comparison across datasets and methods, we report the \textit{relative hypervolume}, defined as the ratio of the achieved hypervolume (shown in blue) to the maximum possible hypervolume enclosed between the utopia and nadir points (shown in gray). The nadir point is set per dataset as the largest objective value observed across all methods plus a margin of 0.05, ensuring all solutions are strictly dominated. The utopia point is task-dependent: for multi-task learning it is set to (0,0), representing zero loss or zero error on all tasks, while for fairness it is set to (min error/min loss, 0), reflecting the best achievable error/loss at zero discrimination. A relative hypervolume of 1 thus corresponds to a solution set that perfectly covers the entire trade-off region between the utopia and nadir points. Note that the utopia point is in general not achievable in multi-objective optimization, as it would minimize all objectives simultaneously.

\subsection{Conditioning on the preference vector}
For the tabular datasets, we concatenate the preference vector $\vec{r}$ with the learned representation $z$ and pass both directly to the preference alignment network, which is sufficient when no spatial structure is present. 
For image data on the other hand, we condition on the preference vector $\vec{r}$ using a Feature-wise Linear Modulation (FiLM) layer~\cite{perez2017visual}, which has been shown to enable effective preference-dependent representation modulation in MOO settings~\cite{chen2022multi, dosovitskiy2019you}. Concretely, a linear layer projects the preference vector $\vec{r}$ to channel-wise scale and shift parameters $(\vec{\gamma},\vec{\beta})$, which modulate the representation $\vec{z}$ channel-wise as $\vec{z}' := \vec{\gamma}(\vec{r}) \odot \vec{z}+\vec{\beta}(\vec{r})$. This requires adding a single linear layer \texttt{Linear($dim(\vec{r})$, $2\cdot dim(\vec{z})$)}, introducing only $2\cdot dim(\vec{z}) \cdot (dim(\vec{\vec{r}}) + 1)$ additional parameters, which are a negligible overhead relative to the full network.

\section{Additional experimental results}
\subsection{Fairness use case}
Figure~\ref{fig:fairness_app} shows the Pareto fronts for all fairness datasets on the training objectives, complementing the main paper results. Table~\ref{tab:fairness_comp_app} reports the computational overhead of each method relative to the Classifier baseline in terms of parameter count, training time, CPU memory, and GPU memory. Both MOSEL variants add substantially fewer parameters than COSMOS across all datasets (+50 vs. +120 for tabular, +1,536 vs. +6,993 for FairFace), while incurring comparable or lower training time and memory overhead. GPU memory is reported for FairFace, tabular datasets were trained on CPU only. 

\begin{table}[ht]
    \centering
    \caption{Computational overhead on fairness datasets relative to classifier (bottom); \# Param.\ are absolute counts, time and CPU are multiplicative factors. Mean~$\pm$~std over five runs.}
    \label{tab:fairness_comp_app}
      \resizebox{\textwidth}{!}{%
    \begin{tabular}{c  c  c  c  c  c}
        \toprule
        Dataset & Method & $\Delta$\#Param. & Time (rel.) & CPU Mem (rel.) & GPU Mem (rel.) \\
        \midrule
        \multirow{3}{*}{Adult}
        & MOSEL-LS & +50  & $1.2 \pm 0.2$ & $1.10 \pm 0.07$ & -- \\
        & MOSEL-CS & +50  & $1.2 \pm 0.1$ & $1.06 \pm 0.06$ & -- \\
        & COSMOS   & +120 & $1.2 \pm 0.2$ & $1.09 \pm 0.03$ & -- \\
        \cmidrule(lr){2-6}
        & Classifier & 6,891 & $44.4 \pm 4.4$\,s & $10.30 \pm 0.34$\,MB & -- \\
        \midrule
        \multirow{3}{*}{COMPAS}
        & MOSEL-LS & +50  & $1.5 \pm 0.7$ & $1.00 \pm 0.03$ & -- \\
        & MOSEL-CS & +50  & $1.5 \pm 0.7$ & $0.96 \pm 0.03$ & -- \\
        & COSMOS   & +120 & $1.1 \pm 0.0$ & $0.99 \pm 0.03$ & -- \\
        \cmidrule(lr){2-6}
        & Classifier & 2,811 & $11.5 \pm 1.0$\,s & $11.43 \pm 0.46$\,MB & -- \\
        \midrule
        \multirow{3}{*}{Default}
        & MOSEL-LS & +50  & $1.2 \pm 0.3$ & $0.91 \pm 0.05$ & -- \\
        & MOSEL-CS & +50  & $1.0 \pm 0.1$ & $0.92 \pm 0.02$ & -- \\
        & COSMOS   & +120 & $1.0 \pm 0.1$ & $0.91 \pm 0.01$ & -- \\
        \cmidrule(lr){2-6}
        & Classifier & 7,011 & $33.2 \pm 5.0$\,s & $12.22 \pm 0.48$\,MB & -- \\
        \midrule
        \multirow{3}{*}{FairFace}
        & MOSEL-LS & +1,536 & $0.7 \pm 0.0$ & $0.88 \pm 0.04$ & $0.50 \pm 0.00$ \\
        & MOSEL-CS & +1,536 & $1.2 \pm 0.6$ & $0.93 \pm 0.09$ & $0.56 \pm 0.07$ \\
        & COSMOS   & +6,993 & $1.6 \pm 0.6$ & $1.05 \pm 0.11$ & $0.58 \pm 0.07$ \\
        \cmidrule(lr){2-6}
        & Classifier & 21,293,393 & $2927.9 \pm 1425.5$\,s & $8.89 \pm 0.04$\,MB & $8929.68 \pm 536.47$\,MB \\
        \bottomrule
    \end{tabular}}
\end{table}

\begin{figure}[ht]
    \centering
    \includegraphics[width=\linewidth]{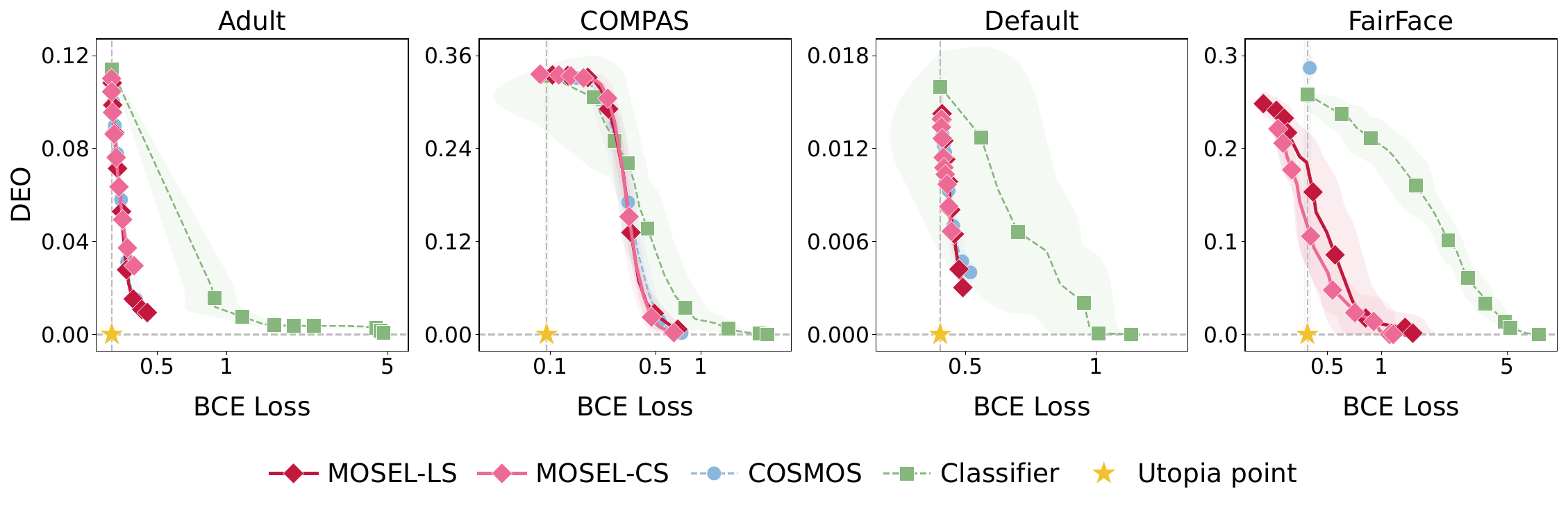}
    \caption{Test-time Pareto fronts on the fairness use case, showing training objectives (BCE loss vs DEO). The grey dotted lines mark zero unfairness and the best BCE loss achievable by an unconstrained classifier; their intersection defines the utopia point, the ideal but infeasible solution, toward which all methods aim to converge.}
    \label{fig:fairness_app}
\end{figure}

\subsection{MTL use case}
Figure~\ref{fig:mtl_app} shows the Pareto fronts for all MTL datasets on the training objectives, complementing the main paper results. Table~\ref{tab:per-task-acc} reports per-task accuracy at the solution closest to the utopia point, and Table~\ref{tab:mtl_app} reports hypervolume and number of Pareto optimal solutions for both loss and error metrics. Both MOSEL variants consistently achieve higher hypervolume than COSMOS and the other baselines across most datasets, with MOSEL-LS performing particularly strongly on the synthetic datasets and MOSEL-CS on CelebA. 

While COSMOS tends to produce more Pareto optimal solutions, MOSEL variants achieve better overall front quality as measured by hypervolume.
Table~\ref{tab:mtl_comp_app} reports computational overhead relative to the Equal weight baseline. Both MOSEL variants add more parameters than COSMOS (+4,320 vs. +708), which is attributable to the FiLM conditioning layer applied to the 720-dimensional representation $z$ in the LeNet architecture. However, this additional parameter count is solely due to preference conditioning and does not translate into meaningful increases in training time or memory consumption, with both remaining comparable to COSMOS and the Equal weight baseline.

\begin{table}[ht]
    \centering
    \caption{Per-task accuracy at the solution closest to the utopia point on the test Pareto front, across all MTL benchmarks. As joint training enables positive transfer, we take $(0, 0)$ as the utopia point and select the solution minimizing Euclidean distance to it, which can lie beyond the single-task baseline. TL and BR denote top-left and bottom-right tasks on the synthetic benchmarks. Mean~$\pm$~std over five independent runs; best results per column shaded.}
    \resizebox{\textwidth}{!}{%
    \begin{tabular}{c cc cc cc cc}
        \toprule
         & \multicolumn{2}{c}{Multi-MNIST} & \multicolumn{2}{c}{Multi-Fashion} & \multicolumn{2}{c}{Multi-Fashion+MNIST} & \multicolumn{2}{c}{CelebA} \\
        \cmidrule(lr){2-3} \cmidrule(lr){4-5} \cmidrule(lr){6-7} \cmidrule(lr){8-9}
        Method & Acc TL & Acc BR & Acc TL & Acc BR & Acc TL & Acc BR & Acc Goatee & Acc Mustache \\
        \midrule
        MOSEL-LS & \cellcolor{gray!30}0.919 $\pm$ 0.001 & \cellcolor{gray!30}0.901 $\pm$ 0.002 & \cellcolor{gray!30}0.829 $\pm$ 0.003 & 0.821 $\pm$ 0.001 & 0.930 $\pm$ 0.001 & \cellcolor{gray!30}0.851 $\pm$ 0.003 & 0.968 $\pm$ 0.002 & 0.964 $\pm$ 0.001 \\
        MOSEL-CS & 0.910 $\pm$ 0.002 & 0.895 $\pm$ 0.002 & 0.817 $\pm$ 0.004 & 0.812 $\pm$ 0.003 & 0.916 $\pm$ 0.003 & 0.845 $\pm$ 0.003 & \cellcolor{gray!30}0.970 $\pm$ 0.001 & \cellcolor{gray!30}0.966 $\pm$ 0.001 \\
        COSMOS & 0.907 $\pm$ 0.003 & 0.886 $\pm$ 0.002 & 0.824 $\pm$ 0.004 & 0.821 $\pm$ 0.004 & 0.915 $\pm$ 0.004 & 0.837 $\pm$ 0.005 & 0.960 $\pm$ 0.005 & 0.959 $\pm$ 0.004 \\
        PCGrad & 0.914 $\pm$ 0.006 & 0.894 $\pm$ 0.004 & 0.814 $\pm$ 0.004 & \cellcolor{gray!30}0.828 $\pm$ 0.001 & 0.928 $\pm$ 0.003 & 0.845 $\pm$ 0.007 & 0.954 $\pm$ 0.000 & 0.961 $\pm$ 0.000 \\
        MGDA & 0.912 $\pm$ 0.009 & 0.897 $\pm$ 0.006 & 0.822 $\pm$ 0.004 & 0.820 $\pm$ 0.002 & 0.919 $\pm$ 0.002 & \cellcolor{gray!30}0.852 $\pm$ 0.002 & 0.967 $\pm$ 0.001 & 0.961 $\pm$ 0.000 \\
        GradNorm & 0.913 $\pm$ 0.005 & 0.897 $\pm$ 0.002 & 0.821 $\pm$ 0.006 & 0.820 $\pm$ 0.006 & 0.914 $\pm$ 0.011 & \cellcolor{gray!30}0.855 $\pm$ 0.004 & 0.967 $\pm$ 0.002 & 0.964 $\pm$ 0.001 \\
        FAMO & 0.909 $\pm$ 0.010 & 0.894 $\pm$ 0.006 & 0.818 $\pm$ 0.005 & 0.818 $\pm$ 0.005 & \cellcolor{gray!30}0.938 $\pm$ 0.002 & 0.838 $\pm$ 0.007 & 0.967 $\pm$ 0.001 & \cellcolor{gray!30}0.965 $\pm$ 0.001 \\
        NashMTL & 0.916 $\pm$ 0.003 & 0.891 $\pm$ 0.006 & 0.824 $\pm$ 0.002 & 0.816 $\pm$ 0.002 & 0.924 $\pm$ 0.011 & 0.846 $\pm$ 0.013 & 0.967 $\pm$ 0.001 & 0.964 $\pm$ 0.001 \\
        Equal weight & 0.913 $\pm$ 0.006 & 0.895 $\pm$ 0.003 & 0.824 $\pm$ 0.007 & 0.820 $\pm$ 0.007 & 0.929 $\pm$ 0.003 & 0.846 $\pm$ 0.007 & 0.963 $\pm$ 0.005 & 0.963 $\pm$ 0.002 \\
        Single Task & 0.914 $\pm$ 0.001 & 0.889 $\pm$ 0.004 & 0.823 $\pm$ 0.005 & 0.811 $\pm$ 0.006 & \cellcolor{gray!30}0.939 $\pm$ 0.003 & 0.841 $\pm$ 0.004 & \cellcolor{gray!30}0.969 $\pm$ 0.002 & 0.964 $\pm$ 0.002 \\
        \bottomrule
    \end{tabular}
    }
    \label{tab:per-task-acc}
\end{table}

\begin{table}[ht]
    \centering
    \caption{Computational overhead on MTL datasets relative to equal weight training (bottom); \\\# Param.\ are absolute counts, time and CPU are multiplicative factors. Mean~$\pm$~std over five runs.}
    \label{tab:mtl_comp_app}
    \resizebox{\textwidth}{!}{%
    \begin{tabular}{c  c  c  c  c  c}
        \toprule
        Dataset & Method & $\Delta$\#Param. & Time (rel.) & CPU Mem (rel.) & GPU Mem (rel.) \\
        \midrule
        \multirow{10}{*}{Multi-MNIST}
        & MOSEL-LS    & +4,320 & $1.1 \pm 0.2$ & $1.15 \pm 0.05$ & $1.01 \pm 0.00$ \\
        & MOSEL-CS    & +4,320 & $1.2 \pm 0.2$ & $1.09 \pm 0.09$ & $1.00 \pm 0.00$ \\
        & COSMOS      & +708   & $1.3 \pm 0.2$ & $1.07 \pm 0.02$ & $1.09 \pm 0.03$ \\
        & PCGrad$^*$  & +0     & $6.2 \pm 0.6$ & $1.22 \pm 0.04$ & $1.00 \pm 0.00$ \\
        & MGDA        & +0     & $2.0 \pm 0.4$ & $1.04 \pm 0.06$ & $1.12 \pm 0.02$ \\
        & GradNorm & +0 & 1.9 $\pm$ 0.2 & 1.07 $\pm$ 0.05 & 1.00 $\pm$ 0.00 \\
        & FAMO & +0 & 1.2 $\pm$ 0.1 & 1.04 $\pm$ 0.04 & 1.00 $\pm$ 0.00 \\
        & NashMTL & +0 & 6.7 $\pm$ 1.0 & 0.76 $\pm$ 0.02 & 1.01 $\pm$ 0.00 \\
        & Single Task$^*$ & +0     & $1.6 \pm 0.1$ & $1.14 \pm 0.04$ & $1.00 \pm 0.00$ \\
        \cmidrule(lr){2-6}
        & Equal weight & 42,350 & $428.8 \pm 75.1$\,s & $14.65 \pm 1.18$\,MB & $63.97 \pm 0.00$\,MB \\
        \midrule
        \multirow{10}{*}{Multi-Fashion}
        & MOSEL-LS    & +4,320 & $1.1 \pm 0.3$ & $1.13 \pm 0.04$ & $0.99 \pm 0.04$ \\
        & MOSEL-CS    & +4,320 & $1.1 \pm 0.1$ & $1.06 \pm 0.06$ & $0.96 \pm 0.09$ \\
        & COSMOS      & +708   & $1.5 \pm 0.2$ & $1.07 \pm 0.08$ & $1.11 \pm 0.04$ \\
        & PCGrad$^*$  & +0     & $7.2 \pm 1.5$ & $1.25 \pm 0.05$ & $1.00 \pm 0.00$ \\
        & MGDA        & +0     & $2.1 \pm 0.4$ & $1.01 \pm 0.04$ & $1.12 \pm 0.02$ \\
        & GradNorm & +0 & 2.0 $\pm$ 0.2 & 1.09 $\pm$ 0.09 & 1.00 $\pm$ 0.00 \\
        & FAMO & +0 & 1.2 $\pm$ 0.1 & 1.10 $\pm$ 0.08 & 1.00 $\pm$ 0.00 \\
        & NashMTL & +0 & 7.1 $\pm$ 1.1 & 0.76 $\pm$ 0.02 & 1.01 $\pm$ 0.00 \\
        & Single Task$^*$ & +0     & $1.7 \pm 0.1$ & $1.14 \pm 0.04$ & $1.00 \pm 0.00$ \\
        \cmidrule(lr){2-6}
        & Equal weight & 42,350 & $401.7 \pm 45.8$\,s & $14.66 \pm 0.90$\,MB & $63.97 \pm 0.01$\,MB \\
        \midrule
        \multirow{10}{*}{Multi-Fashion+MNIST}
        & MOSEL-LS    & +4,320 & $1.4 \pm 0.1$ & $1.01 \pm 0.02$ & $1.01 \pm 0.00$ \\
        & MOSEL-CS    & +4,320 & $1.4 \pm 0.1$ & $0.96 \pm 0.07$ & $1.00 \pm 0.00$ \\
        & COSMOS      & +708   & $1.5 \pm 0.2$ & $0.97 \pm 0.05$ & $1.09 \pm 0.03$ \\
        & PCGrad$^*$  & +0     & $7.4 \pm 1.0$ & $1.15 \pm 0.03$ & $1.00 \pm 0.00$ \\
        & MGDA        & +0     & $1.9 \pm 0.6$ & $0.93 \pm 0.07$ & $1.13 \pm 0.01$ \\
        & GradNorm & +0  & 2.1 $\pm$ 0.2 & 0.96 $\pm$ 0.05 & 1.00 $\pm$ 0.00 \\
        & FAMO & +0  & 1.3 $\pm$ 0.1 & 0.98 $\pm$ 0.06 & 1.00 $\pm$ 0.00 \\
        & NashMTL & +0  & 7.7 $\pm$ 1.3 & 0.71 $\pm$ 0.01 & 1.01 $\pm$ 0.00 \\
        & Single Task$^*$ & +0     & $1.9 \pm 0.0$ & $1.01 \pm 0.05$ & $1.00 \pm 0.00$ \\
        \cmidrule(lr){2-6}
        & Equal weight & 42,350 & $377.7 \pm 55.1$\,s & $15.97 \pm 0.82$\,MB & $63.97 \pm 0.00$\,MB \\
        \bottomrule
        \multicolumn{6}{l}{$^*$PCGrad requires an independent run for each of the five preference vectors, SingleTask requires an independent run for each task}
    \end{tabular}}
\end{table}

\begin{table}[t]
    \centering
    \caption{Relative hypervolume and number of Pareto stationary solutions on the MTL benchmarks, computed on training loss and test error. Mean~$\pm$~std over five runs; best entries per column shaded.}
    \label{tab:mtl_app}
    \resizebox{\textwidth}{!}{%
    \begin{tabular}{c  c cc  cc }
    \toprule
       Dataset & Method & HV Loss $\uparrow$ & HV Error $\uparrow$ & \# PS sol. (Loss)  $\uparrow$ & \# PS sol. (Error)  $\uparrow$ \\
        \midrule
        \multirow{10}{*}{Multi-MNIST} & MOSEL-LS & \cellcolor{gray!30}0.330 $\pm$ 0.010 & \cellcolor{gray!30}0.441 $\pm$ 0.008 & 10.6 $\pm$ 0.8 & 9.0 $\pm$ 0.6 \\
         & MOSEL-CS & 0.318 $\pm$ 0.008 & 0.426 $\pm$ 0.009 & 17.8 $\pm$ 0.4 & \cellcolor{gray!30}16.8 $\pm$ 1.2 \\
         & COSMOS & 0.303 $\pm$ 0.008 & 0.418 $\pm$ 0.007 & \cellcolor{gray!30}20.0 $\pm$ 1.1 & \cellcolor{gray!30}19.0 $\pm$ 3.0 \\
         & PCGrad & 0.274 $\pm$ 0.017 & 0.425 $\pm$ 0.010 & 2.2 $\pm$ 0.7 & 2.8 $\pm$ 0.4 \\
         & MGDA & 0.267 $\pm$ 0.026 & 0.409 $\pm$ 0.034 & 1.0 $\pm$ 0.0 & 1.0 $\pm$ 0.0 \\
         & GradNorm & 0.281 $\pm$ 0.010 & 0.410 $\pm$ 0.016 & 1.0 $\pm$ 0.0 & 1.0 $\pm$ 0.0 \\
         & FAMO & 0.265 $\pm$ 0.029 & 0.394 $\pm$ 0.033 & 1.0 $\pm$ 0.0 & 1.0 $\pm$ 0.0 \\
         & NashMTL & 0.271 $\pm$ 0.015 & 0.401 $\pm$ 0.021 & 1.0 $\pm$ 0.0 & 1.0 $\pm$ 0.0 \\
         & Equal weight & 0.265 $\pm$ 0.008 & 0.406 $\pm$ 0.022 & 1.0 $\pm$ 0.0 & 1.0 $\pm$ 0.0 \\
         & Single Task & 0.264 $\pm$ 0.011 & 0.393 $\pm$ 0.010 & 1.0 $\pm$ 0.0 & 1.0 $\pm$ 0.0 \\
        \midrule
        \multirow{10}{*}{Multi-Fashion} & MOSEL-LS & \cellcolor{gray!30}0.094 $\pm$ 0.004 & 0.171 $\pm$ 0.004 & 10.8 $\pm$ 1.0 & 8.8 $\pm$ 2.9 \\
         & MOSEL-CS & 0.086 $\pm$ 0.004 & 0.158 $\pm$ 0.007 & \cellcolor{gray!30}17.4 $\pm$ 0.8 & \cellcolor{gray!30}16.8 $\pm$ 1.3 \\
         & COSMOS & \cellcolor{gray!30}0.097 $\pm$ 0.006 & \cellcolor{gray!30}0.177 $\pm$ 0.010 & \cellcolor{gray!30}17.8 $\pm$ 0.7 & \cellcolor{gray!30}16.4 $\pm$ 1.9 \\
         & PCGrad & 0.070 $\pm$ 0.007 & \cellcolor{gray!30}0.179 $\pm$ 0.005 & 2.2 $\pm$ 0.7 & 3.6 $\pm$ 0.5 \\
         & MGDA & 0.073 $\pm$ 0.004 & 0.158 $\pm$ 0.007 & 1.0 $\pm$ 0.0 & 1.0 $\pm$ 0.0 \\
         & GradNorm & 0.073 $\pm$ 0.002 & 0.155 $\pm$ 0.004 & 1.0 $\pm$ 0.0 & 1.0 $\pm$ 0.0 \\
         & FAMO & 0.071 $\pm$ 0.004 & 0.150 $\pm$ 0.005 & 1.0 $\pm$ 0.0 & 1.0 $\pm$ 0.0 \\
         & NashMTL & 0.069 $\pm$ 0.004 & 0.154 $\pm$ 0.005 & 1.0 $\pm$ 0.0 & 1.0 $\pm$ 0.0 \\
         & Equal weight & 0.072 $\pm$ 0.009 & 0.161 $\pm$ 0.019 & 1.0 $\pm$ 0.0 & 1.0 $\pm$ 0.0 \\
         & Single Task & 0.072 $\pm$ 0.008 & 0.146 $\pm$ 0.013 & 1.0 $\pm$ 0.0 & 1.0 $\pm$ 0.0 \\
        \midrule
        \multirow{10}{*}{Multi-Fashion+MNIST} & MOSEL-LS & \cellcolor{gray!30}0.295 $\pm$ 0.012 & 0.376 $\pm$ 0.012 & 10.4 $\pm$ 0.8 & 8.6 $\pm$ 5.0 \\
         & MOSEL-CS & 0.281 $\pm$ 0.010 & 0.360 $\pm$ 0.009 & 17.4 $\pm$ 0.5 & \cellcolor{gray!30}19.4 $\pm$ 1.9 \\
         & COSMOS & 0.270 $\pm$ 0.016 & 0.350 $\pm$ 0.016 & \cellcolor{gray!30}22.0 $\pm$ 1.3 & \cellcolor{gray!30}19.0 $\pm$ 1.1 \\
         & PCGrad & 0.263 $\pm$ 0.010 & \cellcolor{gray!30}0.401 $\pm$ 0.008 & 2.4 $\pm$ 0.8 & 3.8 $\pm$ 0.4 \\
         & MGDA & 0.237 $\pm$ 0.009 & 0.354 $\pm$ 0.008 & 1.0 $\pm$ 0.0 & 1.0 $\pm$ 0.0 \\
         & GradNorm & 0.256 $\pm$ 0.011 & 0.355 $\pm$ 0.008 & 1.0 $\pm$ 0.0 & 1.0 $\pm$ 0.0 \\
         & FAMO & 0.244 $\pm$ 0.016 & 0.344 $\pm$ 0.019 & 1.0 $\pm$ 0.0 & 1.0 $\pm$ 0.0 \\
         & NashMTL & 0.238 $\pm$ 0.040 & 0.346 $\pm$ 0.049 & 1.0 $\pm$ 0.0 & 1.0 $\pm$ 0.0 \\
         & Equal weight & 0.245 $\pm$ 0.015 & 0.354 $\pm$ 0.019 & 1.0 $\pm$ 0.0 & 1.0 $\pm$ 0.0 \\
         & Single Task & 0.254 $\pm$ 0.018 & 0.351 $\pm$ 0.010 & 1.0 $\pm$ 0.0 & 1.0 $\pm$ 0.0 \\
        \midrule
        \multirow{10}{*}{CelebA} & MOSEL-LS & 0.555 $\pm$ 0.013 & 0.531 $\pm$ 0.019 & 22.8 $\pm$ 3.0 & 3.2 $\pm$ 1.2 \\
         & MOSEL-CS & \cellcolor{gray!30}0.605 $\pm$ 0.004 & \cellcolor{gray!30}0.556 $\pm$ 0.007 & 8.0 $\pm$ 6.7 & 4.6 $\pm$ 1.7 \\
         & COSMOS & 0.565 $\pm$ 0.010 & 0.536 $\pm$ 0.008 & \cellcolor{gray!30}25.0 $\pm$ 0.0 & \cellcolor{gray!30}13.0 $\pm$ 2.4 \\
         & PCGrad & 0.214 $\pm$ 0.000 & 0.427 $\pm$ 0.000 & 1.8 $\pm$ 0.7 & 1.0 $\pm$ 0.0 \\
         & MGDA & 0.248 $\pm$ 0.089 & 0.510 $\pm$ 0.006 & 1.0 $\pm$ 0.0 & 1.0 $\pm$ 0.0 \\
         & GradNorm & 0.508 $\pm$ 0.012 & 0.525 $\pm$ 0.016 & 1.0 $\pm$ 0.0 & 1.0 $\pm$ 0.0 \\
         & FAMO & 0.506 $\pm$ 0.006 & 0.532 $\pm$ 0.012 & 1.0 $\pm$ 0.0 & 1.0 $\pm$ 0.0 \\
         & NashMTL & 0.285 $\pm$ 0.057 & 0.527 $\pm$ 0.010 & 1.0 $\pm$ 0.0 & 1.0 $\pm$ 0.0 \\
         & Equal weight & 0.420 $\pm$ 0.120 & 0.496 $\pm$ 0.036 & 1.0 $\pm$ 0.0 & 1.0 $\pm$ 0.0 \\
         & Single Task & 0.491 $\pm$ 0.096 & 0.536 $\pm$ 0.021 & 1.0 $\pm$ 0.0 & 1.0 $\pm$ 0.0 \\
        \bottomrule
     \end{tabular}}
\end{table}
\vspace{20pt}
\begin{figure}[H]
    \centering
    \includegraphics[width=\linewidth]{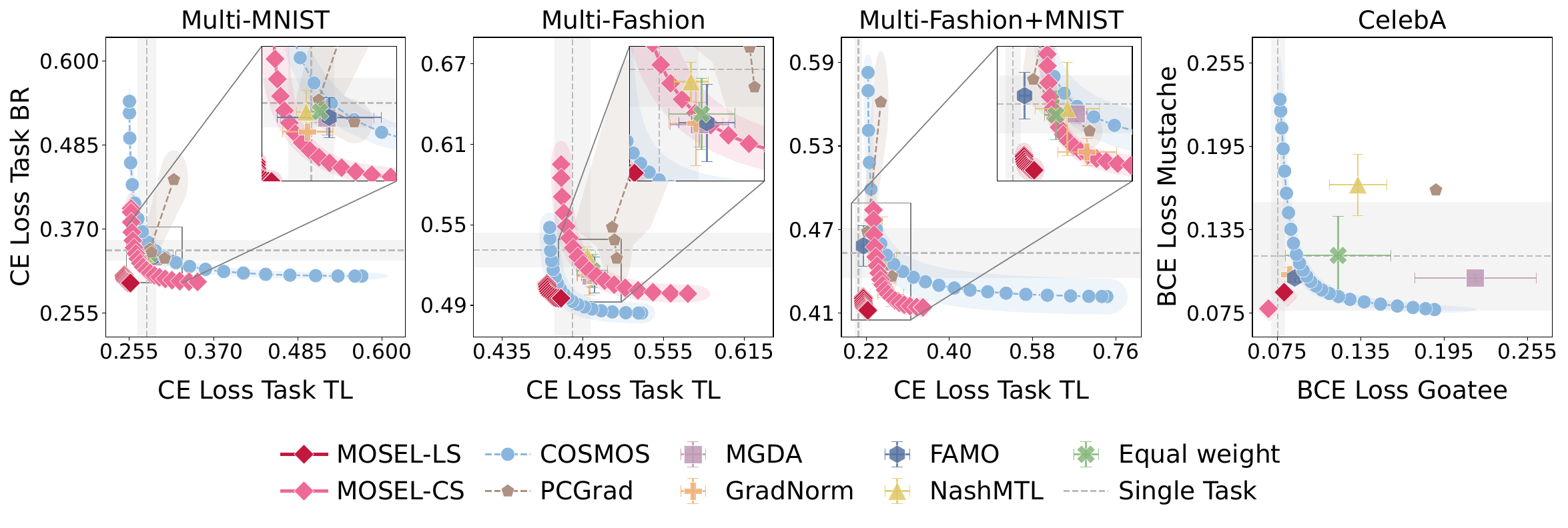}
    \caption{Test-time Pareto fronts on the MTL use case, showing the trade-offs between task-specific (Binary) Cross-Entropy losses. Grey dotted lines indicate best achievable single-task performance on each task; their intersection marks the single-task baseline. In this weakly conflicting setting, positive transfer allows methods to surpass this baseline, with \ourmethod~achieving the largest improvements.}
    \label{fig:mtl_app}
\end{figure}

\newpage
\section{Ablations}
\subsection{Dirichlet parameter $\alpha$}
To isolate the effect of the Dirichlet parameter $\alpha$ on the quality and diversity of the recovered Pareto front, we conduct an ablation study by varying $\alpha$ across four datasets, two from each use case, under both linear scalarization (LS) and Chebyshev scalarization (CS). Results are reported in Table~\ref{tab:alpha}.

The sensitivity to $\alpha$ varies across datasets and scalarization functions. In strongly conflicting settings, performance degrades substantially as $\alpha$ increases, particularly under LS: on Adult, HV drops from $0.939$ at $\alpha=0.5$ to $0.816$ at $\alpha=1.2$, while on COMPAS the decline is even more dramatic (from $0.920$ to $0.164$), indicating a near-total failure to discover diverse trade-off solutions when preference samples concentrate near the simplex center. 
Under CS, Adult follows a similar trend (HV from $0.819$ to $0.524$), whereas COMPAS remains remarkably stable with HV above $0.920$ across all $\alpha$ values. In contrast, the two MTL datasets are largely insensitive to $\alpha$: on {Multi-Fashion}, HV and HV MCR remain effectively constant, and on {Multi-Fashion+MNIST}, only a mild upward trend is observed as $\alpha$ increases. \\
Taken together, these results highlight two key findings. First, the impact of $\alpha$ is most pronounced on datasets where the objectives are strongly conflicting, and where the scalarization function is sensitive to the geometry of the preference simplex (LS more so than CS). Second, even on tasks where hypervolume remains stable, increasing $\alpha$ systematically contracts the extreme solutions, reducing the diversity of the recovered front---an effect that aggregate metrics can obscure. For the MTL setting, this behavior aligns with the theoretical analysis in Appendix~\ref{app:regulariaztion}, which shows that smaller $\alpha$ values (below $1$) increase the variance of stochastic preference sampling, strengthening the implicit bias toward learning features shared across objectives and thereby promoting greater Pareto-front diversity. On the strongly conflicting benchmarks, the effect of $\alpha$ is empirically even more pronounced, yet falls outside the assumptions of our current theoretical framework---extending the analysis to such settings remains an interesting direction for future work. Overall, $\alpha=0.5$ serves as a robust default: it achieves the best or near-best hypervolume across all dataset--scalarization combinations and consistently recovers the widest spread of extreme solutions.

\begin{table}[ht]
    \centering
    \caption{\textbf{Ablation on Dirichlet parameter $\alpha$.} Relative hypervolume on two strongly and two weakly conflicting datasets computed on training loss (HV) and test error (HV MCR). Mean $\pm$ std over five runs. Mean extreme solutions (Obj$_1$-best, Obj$_2$-best) are reported for varying $\alpha$ and scalarization.}
    \label{tab:alpha}
    \resizebox{\linewidth}{!}{
    \begin{tabular}{ccccccc}
         \toprule
         Dataset & Scalarization & $\alpha$ & HV $\uparrow$ & HV MCR $\uparrow$ & Obj$_1$-best (mean) & Obj$_2$-best (mean)  \\
         \midrule
         \multirow{6}{*}{Adult} &  & 0.5 & 0.939 $\pm$ 0.007 & 0.849 $\pm$ 0.012 & (0.143,0.106) & (0.211,0.009) \\
         & LS & 0.8 & 0.912 $\pm$ 0.011 & 0.838 $\pm$ 0.019 & (0.144,0.105) & (0.202,0.014) \\
         &  & 1.2 & 0.816 $\pm$ 0.015 & 0.774 $\pm$ 0.009 & (0.143,0.105) & (0.186,0.030) \\
         \cmidrule{2-7} 
         &  & 0.5 & 0.819 $\pm$ 0.017 & 0.776 $\pm$ 0.015 & (0.142,0.106) & (0.188,0.030) \\
         & CS & 0.8 & 0.702  $\pm$ 0.019 & 0.683 $\pm$ 0.022 & (0.142,0.105) & (0.172,0.049) \\
         &  & 1.2 & 0.524 $\pm$ 0.023 & 0.525 $\pm$ 0.025 & (0.142,0.106) & (0.153,0.079) \\
        \midrule
        \multirow{6}{*}{COMPAS} &  & 0.5 & 0.920 $\pm$ 0.019 & 0.568 $\pm$ 0.008 & (0.021,0.334) & (0.439,0.007) \\
         & LS & 0.8 & 0.541 $\pm$ 0.193 & 0.434 $\pm$ 0.117 & (0.021,0.334) & (0.218,0.169) \\
         &  & 1.2 & 0.164 $\pm$ 0.051 & 0.161 $\pm$ 0.049 & (0.020,0.335) & (0.033,0.324) \\
         \cmidrule{2-7} 
         &  & 0.5 & 0.930 $\pm$ 0.007 & 0.560 $\pm$ 0.005 & (0.021,0.335) & (0.444,0.002) \\
         & CS & 0.8 & 0.927 $\pm$ 0.005 & 0.563 $\pm$ 0.009 & (0.021,0.335) & (0.441,0.005) \\
         &  & 1.2 & 0.921 $\pm$ 0.007 & 0.558 $\pm$ 0.014 & (0.021,0.334) & (0.436,0.008) \\
        \midrule
        \multirow{6}{*}{Multi-Fashion} &  & 0.5 & 0.269 $\pm$ 0.003 & 0.683 $\pm$ 0.003 & (0.169,0.192) & (0.180,0.178) \\
         & LS & 1.0 & 0.267 $\pm$ 0.006 & 0.680 $\pm$ 0.004 & (0.171,0.185) & (0.176,0.180) \\
         &  & 1.2 & 0.268 $\pm$ 0.004 & 0.683 $\pm$ 0.003 & (0.170,0.187) & (0.174,0.178) \\
         \cmidrule{2-7} 
         &  & 0.5 & 0.261 $\pm$ 0.005 & 0.678 $\pm$ 0.006 & (0.172,0.261) & (0.258,0.180) \\
         & CS & 1.0 & 0.260 $\pm$ 0.005 & 0.677 $\pm$ 0.003 & (0.174,0.221) & (0.228,0.180) \\
         &  & 1.2 & 0.260 $\pm$ 0.005 & 0.676 $\pm$ 0.004 & (0.174,0.219) & (0.218,0.181) \\
        \midrule
        \multirow{6}{*}{Multi-F+MNIST} &  & 0.5 & 0.450 $\pm$ 0.011 & 0.789 $\pm$ 0.004 & (0.068,0.165) & (0.081,0.153) \\
         & LS & 1.0 & 0.454 $\pm$ 0.010 & 0.791 $\pm$ 0.005 & (0.069,0.156) & (0.075,0.151) \\
         &  & 1.2 & 0.462 $\pm$ 0.010 & 0.795 $\pm$ 0.005 & (0.068,0.151) & (0.075,0.147) \\
         \cmidrule{2-7} 
         &  & 0.5 & 0.439 $\pm$ 0.012 & 0.782 $\pm$ 0.005 & (0.075,0.212) & (0.225,0.154) \\
         & CS & 1.0 & 0.444 $\pm$ 0.011 & 0.785 $\pm$ 0.006 & (0.074,0.185) & (0.138,0.152) \\
         &  & 1.2 & 0.447 $\pm$ 0.009 & 0.787 $\pm$ 0.004 & (0.075,0.180) & (0.125,0.148) \\
        \bottomrule
    \end{tabular}}
\end{table}

\begin{table}[ht]
    \caption{\textbf{Ablation on split point. } The split is represented by the number of allocated blocks to representation learning (RL) and preference alignment (PA) network respectively. We report obtained HV in training and evaluation space as well as extreme points.}
    \label{tab:split-point}
    \centering
    \begin{tabular}{cccccc}
        \toprule
        Dataset & RL / PA & HV $\uparrow$ & HV MCR $\uparrow$ & Obj$_1$-best & Obj$_2$-best \\
        \midrule
        \multirow{7}{*}{CelebA} & ~3~ / 35 & 0.213 & 0.396 & (0.075,0.080) & (0.076,0.079) \\
        & ~7~ / 31 & 0.238 & 0.406 & (0.070,0.077) & (0.071,0.076) \\
        & 11 / 27 & 0.232 & 0.401 & (0.070,0.079) & (0.071,0.078) \\
        & 17 / 21 & 0.242 & 0.415 & (0.068,0.079) & (0.071,0.076) \\
        & 23 / 15 & 0.242 & 0.415 & (0.068,0.077) & (0.069,0.077) \\
        & 31 / ~7~ & 0.238 & 0.422 & (0.069,0.078) & (0.071,0.076) \\ 
        & 33 / ~5~ & 0.222 & 0.402 & (0.072,0.086) & (0.077,0.079) \\
        \midrule
        \multirow{4}{*}{FairFace}  & 5 / 6 & 0.595 & 0.562 & (0.264,0.239) & (0.720,0.016) \\
        & 6 / 5 & 0.616 & 0.561 & (0.243,0.236) & (0.686,0.008) \\
        & 7 / 4 & 0.627 & 0.541 & (0.235,0.235) & (0.728,0.000) \\
        & 8 / 3 & 0.627 & 0.587 & (0.202,0.255) & (0.481,0.027) \\
        \bottomrule
    \end{tabular}
\end{table}

\begin{table}[t]
    \centering
    \caption{\textbf{Ablation on learning rate ratio. } Sensitivity of MOSEL to the follower/leader learning rate ratio on Adult. The leader learning rate is fixed at $\eta_{\text{leader}} = 0.001$. The default ratio is marked in gray; best values per column are in bold. MCR-best and DEO-best report the values of the solution that is best on the respective objective.}
    \label{tab:lr_ratio}
    \small
    \begin{tabular}{rcccccc}
        \toprule
        Ratio & $\eta_{\text{leader}}$ & $\eta_{\text{follower}}$ & HV $\uparrow$ & HV MCR $\uparrow$ & MCR-best & Obj$_2$DEO-best \\
        \midrule
        $2\times$   & 0.001 & 0.002 & 0.933          & 0.849          & (0.143, 0.109) & (0.209, 0.011) \\
        $5\times$   & 0.001 & 0.005 & \textbf{0.939} & \textbf{0.853} & (0.143, 0.108) & (0.211, 0.010) \\
        \rowcolor{gray!15}
        $10\times$  & 0.001 & 0.010 & 0.935          & 0.845          & (0.143, 0.106) & (0.211, 0.010) \\
        $20\times$  & 0.001 & 0.020 & \textbf{0.939} & \textbf{0.853} & (0.143, 0.107) & (0.211, 0.009) \\
        $50\times$  & 0.001 & 0.050 & 0.876          & 0.807          & (0.145, 0.104) & (0.196, 0.020) \\
        $100\times$ & 0.001 & 0.100 & 0.755          & 0.715          & (0.146, 0.098) & (0.179, 0.040) \\
        \bottomrule
    \end{tabular}
\end{table}

\subsection{Split point}
To further analyze the impact of the split point on performance, we conducted ablations on ResNet, EfficientNet, and DistilBERT (see App.~\ref{app:transformer}).
The results for ResNet and EfficientNet for FairFace and CelebA respectively, are shown in Table~\ref{tab:split-point}.

On CelebA (EfficientNet-B4, 38 total blocks), performance varies primarily at the two extremes of the allocation spectrum. Assigning nearly all capacity to the preference alignment network (3\,/\,35) yields the lowest hypervolume in both training and evaluation, as the representation learning network lacks sufficient depth to learn transferable representations. 

Conversely, restricting the preference alignment network to only 5 blocks (33\,/\,5) also degrades performance, as the preference alignment network can no longer adequately specialize to diverse preference vectors. Between these extremes, performance remains stable: splits ranging from 7\,/\,31 to 31\,/\,7 all achieve comparable hypervolume, indicating broad robustness to the precise split location.
On FairFace (ResNet-34, 11 total blocks), we observe a similar pattern. Hypervolume increases modestly from $0.595$ at split 5\,/\,6 to $0.627$ at 7\,/\,4 and 8\,/\,3, the performance differences across all four tested splits remain small. %

Across all architectures, performance is robust once the preference alignment network has sufficient capacity, indicating that the split point is not a critical hyperparameter. Based on these results, we recommend allocating approximately the final 20–25\% of the blocks to the preference alignment network, while assigning the remaining blocks to the representation learning network. This simple guideline consistently performs well across CNNs and Transformers, substantially reducing the need for extensive split point tuning. 

\begin{table}[t]
    \centering
    \caption{Results on DistilBERT text-classification. Results show HV and extreme points for the front for different split points (splits between transformer blocks)}
    \label{tab:transformer}
    \resizebox{\textwidth}{!}{%
    \begin{tabular}{cccccccc}
        \toprule
         & & & \multicolumn{2}{c}{Fairness extreme} & & \multicolumn{2}{c}{Accuracy extreme} \\
        \cmidrule{4-5} \cmidrule{7-8}
         RL / PA & HV $\uparrow$& HV MCR $\uparrow$ & Acc@MinDEO(\%) $\uparrow$ & MinDEO $\downarrow$ & & DEO@MaxAcc $\downarrow$ & MaxAcc(\%)$\uparrow$ \\
         \midrule
        5 / 1 & 0.423 & 0.479 & 88.68 & 0.0000 & & 0.037 & 92.22 \\
        4 / 2 & 0.419 & 0.477 & 88.69 & 0.0001 & & 0.035 & 92.17 \\
        3 / 3 & 0.419 & 0.475 & 88.70 & 0.0002 & & 0.036 & 92.12 \\
        \bottomrule
    \end{tabular}
    }\vspace{-10pt}
\end{table}

\subsection{Learning rate ratio}
To analyze the impact of the learning rate ratio, we keep the leader (preference alignment) network at the standard learning rate and vary the follower/leader ratio on the example of the Adult dataset. Table~\ref{tab:lr_ratio} shows that MOSEL is robust for ratios between 2× and 20×, with degradation only at 50×+, where an overly aggressive follower destabilizes training. Our default 10× ratio lies comfortably within this robust regime and consistently performs well without additional tuning.

\subsection{Number of objectives}
To validate the scalability of our method to many objectives, we conduct an ablation study on CelebA as presented in Figure~\ref{fig:celeba}. To this end, we identify the most correlated (\texttt{'Heavy Makeup'} and \texttt{'Wearing Lipstick'}) and least correlated (\texttt{'Pale Skin'} and \texttt{'Wearing Necklace'}) task pairs from the full set of 40 CelebA attributes. Starting from each pair, we iteratively add the next most correlated or least correlated task, scaling up to 10 objectives per set. Task correlation is determined via Pearson correlation coefficients $\rho$ computed over the attribute labels, greedily selecting tasks that maximize (i.e., $|\rho|\approx1$) or minimize (i.e., $|\rho|\approx0$) the mean absolute correlation with the already selected set. All configurations are trained following the same pipeline as the main CelebA experiments. Since hypervolume computation scales exponentially with the number of objectives, we restrict evaluation to the closest point to utopia,
\begin{wrapfigure}[14]{l}{0.5\textwidth}
    \centering
    \includegraphics[width=\linewidth]{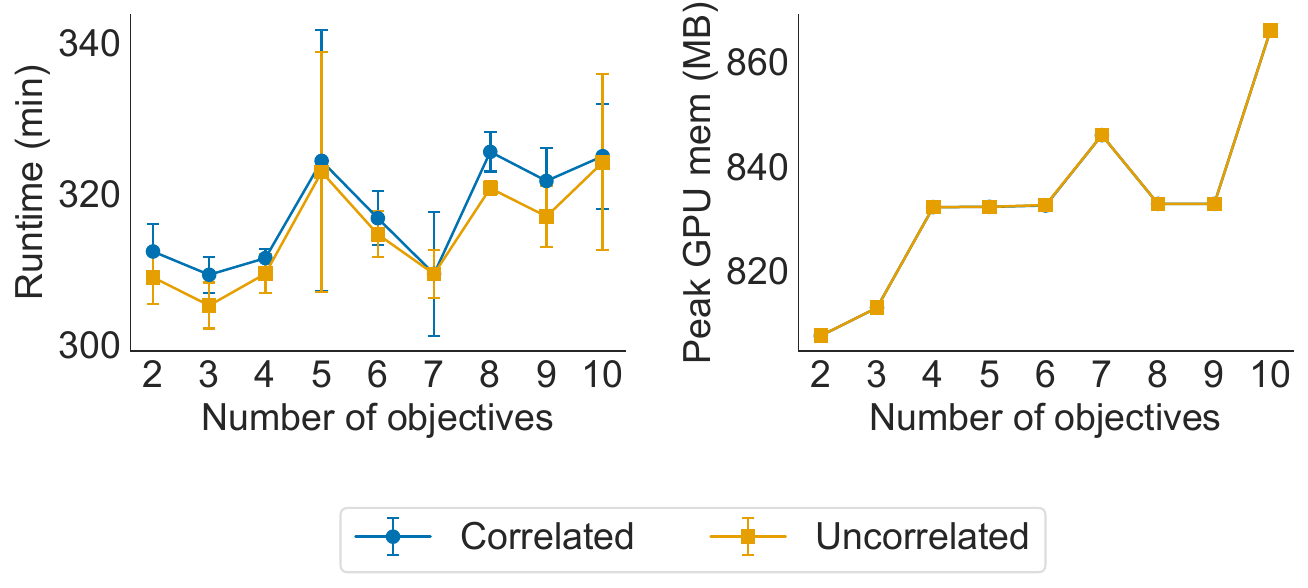}
    \caption{Ablation study on CelebA: runtime and peak memory for increasing numbers of objectives, for correlated and uncorrelated features.}
    \label{fig:runtime}
\end{wrapfigure}
which is generally the solution of interest in MTL. Mean accuracy and improvement over the single task baseline over 5 seeds are reported in Figure~\ref{fig:celeba}.
Note that for improved readability, some attribute names in Figure~\ref{fig:celeba} have been abbreviated without altering their meaning.\\
Figure~\ref{fig:runtime} reports runtime and peak memory for correlated and uncorrelated features. We can see that memory is identical for both settings and only slightly increases with the numbers of objectives. While small differences in runtime exist between 
correlated and uncorrelated settings, the overall runtime 
remains roughly constant across the number of objectives.

\section{Extension to transformer-based architectures}
\label{app:transformer}
To explore the applicability of our method beyond MLPs and CNNs, we additionally conducted experiments on transformer-based architectures. 
To this end, we evaluate MOSEL on fairness-aware text classification using DistilBERT~\citep{sanh2019distilbert} on a 100k-sample subset of CivilComments~\citep{jigsaw-unintended-bias-in-toxicity-classification}, optimizing toxicity prediction (BCE) and fairness (DEO) under linear scalarization with Dirichlet parameter $\alpha=0.5$. DistilBERT has six Transformer blocks, and we evaluate three candidate split points.

MOSEL naturally extends to this Transformer architecture without any architectural modifications, achieving consistent Pareto-front quality across three different split points (HV 0.419–0.423, standard deviation $\leq$ 0.011). These results emphasize, that MOSEL does not rely on architecture-specific decompositions, but instead leverages the modular structure common to modern neural networks by partitioning them at their natural block boundaries (e.g., residual stages or Transformer blocks). Extending this empirical validation to foundation vision-language models and LLMs is an exciting direction for future work.

\newpage
\section{Convergence analysis}
To verify that preference conditioning does not introduce optimization instability, we compare the convergence behavior of MOSEL-LS against an equal weight baseline on Multi-MNIST. Both methods use an identical optimizer and learning rate schedule (standard Adam). For each epoch
we record the maximum loss across all sampled preference vectors, capturing worst-case
optimization behavior. Results are reported in Table~\ref{tab:convergence} as means
and standard deviations over 5 seeds.

Both methods exhibit nearly identical convergence profiles: a rapid descent during the first 10 epochs followed by a stable plateau, with no signs of divergence or increasing instability. After the initial transient, MOSEL maintains a small, constant gap in maximum loss relative to the baseline ($\Delta \approx +0.09$ from epoch 10 onward).
This gap is expected, as MOSEL repeatedly optimizes under extreme preference vectors (encouraged by our choice of $\alpha < 1$) in order to recover a diverse Pareto front, rather than a single fixed weighting. Despite this more challenging optimization setting, MOSEL achieves tighter cross-seed variance than the baseline ($\pm 0.02$ vs.\ $\pm 0.04$ at convergence), and the constant magnitude of $\Delta$ after epoch 10 confirms that the additional complexity introduced by preference conditioning does not compound over training.
These results demonstrate that MOSEL preserves the optimization stability of conventional single-objective training while successfully learning to cover the Pareto front.

\begin{table}[h]
    \centering
    \caption{Per-epoch maximum objective loss (mean $\pm$ std over 5 seeds). $\Delta$ is defined as difference between MOSEL and Equal Weight baseline; negative values indicate lower (better) worst-case loss for MOSEL. After an initial transient, MOSEL maintains a small, constant gap with tighter cross-seed variance, confirming stable optimization throughout training.}
    \label{tab:convergence}
    \begin{tabular}{cccc}
        \toprule
        Epoch & MOSEL-LS Max & Equal Weight Max & $\Delta$ Max \\
        \midrule
        0 & 5.54 $\pm$ 0.86 & 15.89 $\pm$ 5.91 & -10.35 \\
        10 & 0.52 $\pm$ 0.03 & 0.42 $\pm$ 0.04 & +0.09 \\
        30 & 0.38 $\pm$ 0.02 & 0.30 $\pm$ 0.05 & +0.08 \\
        50 & 0.36 $\pm$ 0.02 & 0.28 $\pm$ 0.05 & +0.08 \\
        70 & 0.37 $\pm$ 0.01 & 0.28 $\pm$ 0.05 & +0.09 \\
        99 & 0.36 $\pm$ 0.02 & 0.27 $\pm$ 0.04 & +0.09 \\
         \bottomrule
    \end{tabular}
\end{table}

\newpage

\end{document}